\documentclass{article}

\usepackage{arxiv}

\usepackage[utf8]{inputenc} 
\usepackage[T1]{fontenc}    
\usepackage{hyperref}       
\usepackage{url}            
\usepackage{booktabs}       
\usepackage{amsfonts}       
\usepackage{nicefrac}       
\usepackage{microtype}      
\usepackage{lipsum}
\usepackage{graphicx}

\usepackage{amsmath}
\usepackage{bm}
\usepackage{booktabs}
\usepackage{multirow}
\usepackage{pifont}
\usepackage{bbm}
\usepackage{subcaption}
\usepackage{float}
\usepackage{enumitem}
\usepackage{amssymb}
\usepackage{xcolor}
\usepackage[blocks]{authblk}

\newtheorem{assumption}{Assumption}
\newtheorem{proposition}{Proposition}
\newtheorem{lemma}{Lemma}
\newtheorem{proof}{Proof}
\newtheorem{theorem}{Theorem}
\newtheorem{corollary}{Corollary}
\newcommand{\E}{\mathbb{E}}

\title{MARCO: Click-Intent Decomposition for Calibrated Ads Conversion Prediction}

\author[1]{Shiwen Shen}
\author[1]{Xiru Huang}
\author[1]{Liang Luo}
\author[1]{Jianbo Sun}
\author[1]{He Lyu}
\author[1]{Zihang Fu}
\author[1]{Ivonne Xu}
\author[1]{Zhizhuo Li}
\author[1]{Zhengyu Zhang}
\author[1]{Pei-Ju Sung}
\author[1]{Yunmiao Wang}
\author[1]{Zixuan Wang}
\author[1]{Zhengli Zhao}
\author[1]{Qiang Jin}
\author[1]{Mike Jermann}
\author[1]{Mingda Li}
\author[1]{Yang Xiao}
\author[1]{Bhavana Challa}
\author[1]{Brooke Bian}
\author[1]{Yang Li}
\author[1]{Ashish Chamoli}
\author[1]{Bibek Bhusal}
\author[1]{Danning Di}
\author[1]{Yuan Jin}
\author[1]{Meet Raval}
\author[1]{Zhiwen Chen}
\author[1]{Boyao Sun}
\author[1]{Shuguang Wang}
\author[1]{Yunlong He}
\author[1]{Yantao Yao}
\author[1]{Sagar Chordia}
\author[1]{Wenlin Chen}
\author[1]{Santanu Kolay}
\author[1]{Qin Huang}
\author[1]{Ellie Wen}

\affil[1]{Meta AI, Menlo Park, California, USA}

\begin{document}
\maketitle

\begin{abstract}
Not all clicks are equal. Industrial ads ranking decouples conversion probability into click-through rate (CTR) and post-click conversion rate (CVR), yet treats every click as the same event. 
In reality, users provide a free, self-generated signal of intent through their physical UI interactions. Different click types on the same ad exhibit a $4\times$ difference in actual conversion rates. By conflating these signals, the standard CVR model systematically under-predicts high-intent clicks and over-predicts low-intent ones, which is a severe bias masked by near-perfect aggregate calibration.

We propose MARCO (\textbf{M}ulti-intent \textbf{A}ds \textbf{R}anking \textbf{C}omposition \textbf{O}ptimization), a framework that resolves this bias by decomposing each click by intent. 
Using the logged click type as a free behavioral label, MARCO trains per-intent CVR heads on homogeneous populations, and at serving time composes their per-intent CVR estimates under a predicted distribution over intents. 
Theoretically, we prove that decomposition never raises population risk, give the exact headroom under squared loss and non-negativity under the deployed loss, and show through a routing-efficiency dial how much of it reaches serving. Because the population-optimal score is unchanged, any gain is a finite-capacity estimation and calibration effect that we validated both offline and online.
For deployment at scale, we further cast multi-impression, multi-click attribution as credit assignment with a bias-variance tradeoff analogous to RL return estimation, showing last-impression, first-click attribution is the low-bias, low-variance, deterministic choice under production constraints, and derive three consistency conditions enforced end-to-end at scale.

Deployed at binary intent granularity, MARCO corrects per-intent calibration to approximately $100\%$, lifts conversions per click by $+2.80\%$, and drives $+0.98\%$ cumulative improvement in topline metrics. 
\end{abstract}


\section{Introduction}
\label{sec:intro}

Industrial ads ranking systems estimate impression conversion probability by factorizing it into a click-through rate (CTR) and a post-click conversion rate (CVR) as an exact marginal factorization~\cite{ma2018esmm,wen2020esm2}:
\begin{equation}
\label{eq:standard}
\Phi_{\text{std}}(x) = \lambda(x) \cdot \mu(x),
\end{equation}
where $x$ denotes impression-time features, $\lambda(x) := P(\text{click} \mid \text{imp}, x)$ is the CTR, $\mu(x) := P(\text{conv} \mid \text{click}, x)$ is the CVR, and $\Phi_{\mathrm{std}}(x) := P(\text{conv} \mid \text{imp}, x)$ is the conversion score used for auction ranking. Each component is estimated by a dedicated machine learning model—a compositional approach that has served as the industry standard for over a decade.

The primary operational challenge lies in CVR estimation: predicting conversion rates across a mixed click population is difficult due to substantial variance across interaction types. High-intent actions convert at elevated rates, whereas low-intent actions convert far less frequently. Compressing this heterogeneous click mixture into a single scalar prediction introduces group-conditional bias that no finite-capacity model can resolve. Conversely, partitioning clicks into homogeneous intent strata improves estimation efficiency, leveraging foundational principles of post-stratified estimation~\cite{holtsmith1979poststrat,little1993poststrat}. Unlike classical post-stratification, the intent stratum in ad ranking is a post-impression outcome unobserved during scoring. We address this by predicting the intent distribution at inference time and routing per-stratum estimates through the predicted distribution.

Standard interaction logs already contain the signal required to decouple user intent. Modern social ads feature distinct UI surfaces: call-to-action (CTA) taps navigate off-platform, whereas social interactions (e.g., likes, comments) reflect lightweight on-platform engagement. Conventional architectures collapse these behaviors into a single click label, leaving a zero-cost supervision signal unexploited. MARCO instead leverages logged click types as free supervision for intent. Empirically, high-intent CTA taps convert at nearly $4\times$ the rate of social actions—a gap single-head models cannot capture.

Ignoring this signal forces a single CVR head to fit distinct conversion funnels, causing systematic subgroup \emph{miscalibration}: under-predicting high-intent traffic while over-predicting low-intent traffic. This directly degrades ranking, as inflated low-intent scores displace higher-converting candidates in the auction. Because these opposing errors cancel in aggregate, standard monitoring tools report healthy overall calibration while masking severe internal distortion. We term this silent failure mode \emph{click-intent heterogeneity} and formalize its observability limit in Section~\ref{sec:problem}.

While prior literature addresses adjacent modeling challenges, no existing framework restructures the CTR-CVR composition formula to isolate heterogeneous click populations. Vertical funnel methods~\cite{ma2018esmm,wen2020esm2,wang2022escm2,dcmt2023,xi2021aitm} decompose the user journey after the click. Multi-task learning and mixture-of-experts (MoE) architectures \cite{ma2018mmoe,tang2020ple,chang2023pepnet} optimize routing across shared experts, yet ultimately blend representations into a single prediction over the mixed click population. Expert gating modifies architectural capacity rather than the supervised target, leaving the underlying label distribution unaltered for each task. Post-hoc calibration techniques~\cite{platt1999probabilistic,zadrozny2002transforming,calmatters2023,beyondpartitions2023} cannot condition on click intent, as the specific interaction type is a post-impression event unobservable at inference time. While intent models~\cite{netflix2024intent} enrich feature representations, they leave the underlying prediction composition unchanged.

We propose \textbf{MARCO} (\textbf{M}ulti-intent \textbf{A}ds \textbf{R}anking \textbf{C}omposition \textbf{O}ptimization), a framework that resolves click-intent heterogeneity through structural intent decomposition. MARCO partitions click events into intent-differentiated sub-categories, estimating dedicated CTR and CVR predictions over more homogeneous sub-populations. At training time, MARCO uses the logged click type as a free supervision label; at serving time, it predicts the latent intent distribution across categories. By reformulating the prediction composition, MARCO eliminates group-conditional miscalibration while remaining orthogonal to and fully composable with existing CVR architectures. Our contributions are:
\begin{enumerate}[leftmargin=*]
\item \textbf{Problem Formulation \& Observability Limit.} We identify \emph{click-intent heterogeneity} as a root cause of systematic miscalibration that conventional architectures cannot eliminate, and prove that a single impression-time prediction is inherently incapable of achieving per-intent calibration (\S\ref{sec:problem}).
\item \textbf{Free Supervision \& Mechanism Isolation.} We leverage logged click types as zero-cost behavioral supervision labels. Through controlled ablations, we demonstrate that single-output architectures consistently fail to eliminate per-intent miscalibration, even when augmented with auxiliary intent tasks, historical features, or Mixture-of-Experts (MoE) gating. Conversely, exposing per-intent outputs via output-layer decomposition eliminates over $99\%$ of calibration error with zero annotation cost and negligible parameter overhead (\S\ref{sec:experiments}).

\item \textbf{Production System \& Industrial Impact.} We articulate the end-to-end system design, including a credit-attribution design and the cross-pipeline consistency conditions the system must satisfy (\S\ref{sec:attribution}). Deployed at scale across Meta's ad platforms under binary intent granularity, MARCO increases conversions per click by $+2.80\%$ in a production holdback and delivers a $+0.98\%$ cumulative topline metric lift across successive launches (\S\ref{sec:experiments}).
\item \textbf{Theoretical Foundation \& Routing Efficiency.} We formalize intent decomposition as a population-level class enlargement, proving weak dominance over pooled models and establishing generic strictly positive headroom $\Delta_{\mathcal{F}}$. Furthermore, we introduce the routing-efficiency parameter $\eta$ to quantify the fraction of theoretical headroom realized at inference (\S\ref{sec:theory}).
\end{enumerate}

The remainder of this paper is organized as follows. Section~\ref{sec:problem} formalizes the click-intent heterogeneity problem. Section~\ref{sec:theory} establishes the theoretical framing and headroom analysis for MARCO. Section~\ref{sec:method} details the proposed MARCO model, while Section~\ref{sec:attribution} describes the credit attribution design and end-to-end system architecture. Section~\ref{sec:experiments} presents our offline evaluations and online A/B testing results. All mathematical proofs are deferred to Appendix~\ref{sec:appendix}.

\begin{figure*}[t]
\centering
\includegraphics[width=\textwidth]{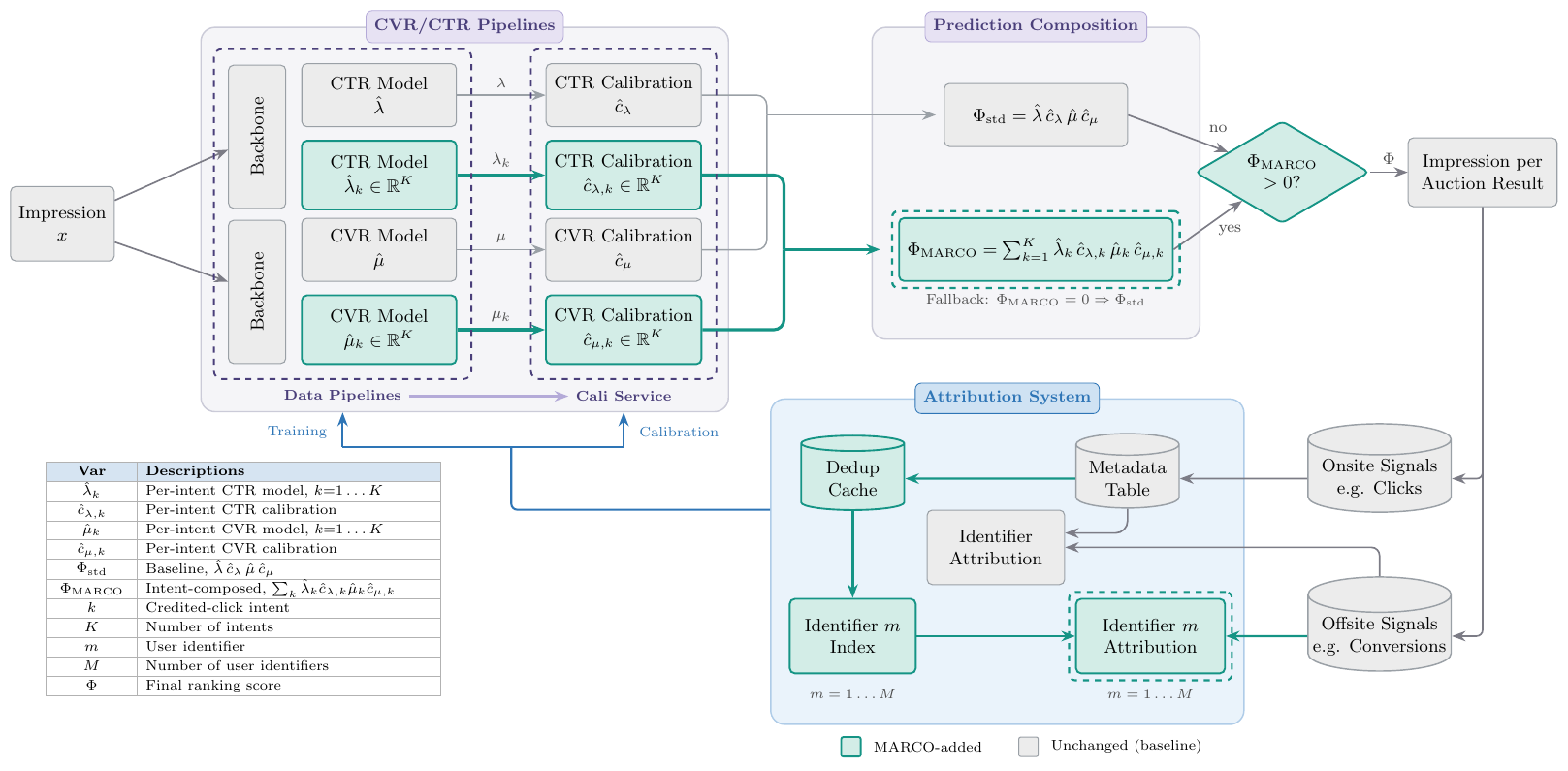}
\caption{The MARCO ads-ranking system. Gray marks the unchanged baseline and teal marks what MARCO adds.}
\label{fig:system}
\end{figure*}

\section{Click-Intent Heterogeneity in Ads Ranking}
\label{sec:problem}

Recall from Eq.~\ref{eq:standard} that conventional ads ranking scores an impression via $\Phi_{\mathrm{std}}(x) = \lambda(x) \cdot \mu(x)$, relying on a single pooled conversion head $\mu(x)$ across all click types. To capture this intent variation, we partition the click space into $K$ mutually exclusive and collectively exhaustive intent categories covering all click types, indexed by $k \in \{1,\dots,K\}$, grouped by historical click-type conversion likelihood. Intent routing relies strictly on logged interaction types rather than conversion outcomes, eliminating target leakage.

For impression features $x$, we define the per-intent CTR $\lambda_k(x) := P(\text{click}_k \mid \text{imp}, x)$, per-intent CVR $\mu_k(x) := P(\text{conv} \mid \text{click}_k, x)$, and intent distribution $\pi_k(x) := P(k \mid \text{click}, x)$. By the law of total probability, the overall impression conversion probability estimated by MARCO is:
\begin{equation}\label{eq:marco}
\Phi_{\mathrm{MARCO}}(x) = \sum_{k=1}^{K} \lambda_k(x) \mu_k(x).
\end{equation}
While $\Phi_{\mathrm{MARCO}}(x)$ mathematically generalizes to arbitrary $K$, our evaluation and deployment focus on the binary instantiation ($K=2$), capturing the dominant contrast between off-platform CTA taps and on-platform social engagements. \emph{Throughout the remainder of this paper, we drop the explicit feature argument $(x)$ for brevity whenever feature conditioning is clear from context.}

We evaluate prediction quality using the \emph{calibration ratio} (predicted over observed conversion rate). Because a pooled head $\mu$ outputs a single scalar across heterogeneous clicks, opposing errors from under-predicting high-intent traffic and over-predicting low-intent traffic cancel in aggregate, masking severe subgroup bias behind a deceptively healthy overall ratio.

Existing ads ranking architectures (the gray baseline path in Figure~\ref{fig:system}) are structurally blind to per-intent bias across every ranking stage:
\begin{enumerate}[leftmargin=*]
    \item \textbf{Attribution System:} Clicks and conversions are joined without distinguishing interaction types, generating pooled training targets.
    \item \textbf{Model Training:} CTR and CVR heads optimize a single loss over mixed click traffic, ignoring intent differentiation.
    \item \textbf{Calibration Service:} Calibration relies strictly on impression-time features, leaving unobserved post-impression intent miscalibrated beneath aggregate cancellation.
    \item \textbf{Prediction Composition:} Pooled estimates are composed via Eq.~\ref{eq:standard}, propagating single-click bias into auction scores.
\end{enumerate}

This limitation represents a structural observability barrier rather than a model capacity constraint. As long as a ranking system outputs a single scalar prediction $\mu$ prior to interaction observation, no added capacity, architectural gating, or post-hoc recalibration can achieve simultaneous group-conditional calibration. We formalize this fundamental limit in Proposition~\ref{prop:observability}.

\begin{proposition}[Per-intent calibration is unattainable by a single output before the click occurs]
\label{prop:observability}
Before the click intent is observed, under a proper scoring rule and a fixed impression-time feature set, no model emitting a single CVR value $\mu$ can be simultaneously calibrated for all intent categories whenever the per-intent CVRs $\mu_k$ differ. It equals at most one $\mu_k$.
\end{proposition}

Section~\ref{sec:theory} formalizes this mechanism as a class-enlargement headroom, proving it is non-negative and generically strictly positive. Section~\ref{sec:experiments} directly quantifies this headroom on held-out benchmarks.

\section{Theoretical Analysis}\label{sec:theory}

We frame intent decomposition as a population-level \emph{class enlargement} ($\mathcal{F} \subseteq \mathcal{G}$) and establish three key theoretical guarantees:
\begin{enumerate}[leftmargin=*, label=(\arabic*)]
    \item \textbf{Weak Dominance:} The intent-decomposed class never yields higher risk than the pooled class at the population optimum (Theorem~\ref{app:respgate}(i)).
    \item \textbf{Genericity:} Strict improvement over pooled models occurs almost surely across parameter space (Proposition~\ref{prop:genericity}).
    \item \textbf{Routing Monotonicity:} Expanding CTR model capacity provably realizes a larger fraction of the theoretical headroom (Proposition~\ref{prop:routing_monotonicity}).
\end{enumerate}
These guarantees characterize the best achievable risk within each hypothesis class. Because the pooled and decomposed scores coincide at the true Bayes optimum ($\Phi_{\mathrm{MARCO}} = \Phi_{\mathrm{std}}$, Lemma~\ref{prop:bayes_equivalence}), any realized gain is strictly a finite-capacity estimation and calibration effect.

\paragraph{Class Enlargement and Headroom ($\Delta_{\mathcal{F}}$).}
Let $\mathcal{F}$ denote the hypothesis class of a single pooled conversion head, and let $\mathcal{G} = \{ g = \sum_{k=1}^K \pi_k f_k : f_k \in \mathcal{F} \}$ denote the intent-decomposed class. We define the theoretical headroom opened by decomposition as the risk gap between the best-in-class pooled model and the best-in-class decomposed model:
\begin{equation}\label{eq:headroom}
\Delta_{\mathcal{F}} := \inf_{f \in \mathcal{F}} \mathcal{R}(f) - \inf_{g \in \mathcal{G}} \mathcal{R}(g)
\end{equation}
Because $\mathcal{F} \subseteq \mathcal{G}$, weak dominance ($\Delta_{\mathcal{F}} \ge 0$) holds universally regardless of backbone architecture, feature representations, or loss function (Theorem~\ref{app:respgate}(i)). Under squared loss, $\Delta_{\mathcal{F}}$ resolves to an exact $L^2$ projection energy $\| P_{\mathcal{G}}\mu - P_{\mathcal{F}}\mu \|^2 \ge 0$ (Proposition~\ref{prop:exact_headroom}). Under the deployed Normalized Entropy metric, while non-linearity prevents a closed-form expression, weak dominance ($\Delta_{\mathcal{F}} \ge 0$) holds unconditionally. 

Furthermore, Proposition~\ref{prop:genericity} establishes that parameter configurations yielding zero headroom ($\Delta_{\mathcal{F}} = 0$) form a set of Lebesgue measure zero. Thus, strict headroom ($\Delta_{\mathcal{F}} > 0$) is generic, i.e., decomposition strictly improves population risk whenever per-intent conversion rates differ. In Section~\ref{sec:experiments}, we directly measure this headroom on held-out production traffic at $K=2$.

\paragraph{Routing Efficiency and CTR Capacity.}
The theoretical headroom $\Delta_{\mathcal{F}}$ is a CVR-side quantity. How much of this potential reaches inference depends on the intent distribution $\hat{\pi}_k = \hat{\lambda}_k / \hat{\lambda}$ predicted by the CTR model. We formalize this transfer via the \emph{routing efficiency}:
\begin{equation}\label{eq:routing_efficiency}
\hat{\eta} := 1 - \frac{\varepsilon_{\mathrm{route}}}{\Delta_{\mathcal{F}}}
\end{equation}
where $\varepsilon_{\mathrm{route}}$ represents the heterogeneity-weighted routing cost incurred by imperfect intent predictions (Lemma~\ref{lem:routing_cost}). The routing cost obeys $\varepsilon_{\mathrm{route}} \le \mathbb{E}[ \|\hat{\pi} - \pi\|^2 s^2 ]$, where $s^2 = \sum_k (\mu_k - \mu)^2$ measures intent variance. This cost vanishes where intent is homogeneous ($s=0$) and concentrates strictly where per-intent conversion rates diverge. 

In Proposition~\ref{prop:routing_monotonicity}, we prove that expanding CTR model capacity expands the realizable routing class $\Pi^{(c)}$, monotonically reducing routing error $\varepsilon_{\mathrm{route}}^*$ and increasing serving efficiency $\hat{\eta}$. Consequently, larger CTR models unlock a higher fraction of the CVR headroom—a theoretical property directly validated by our offline capacity sweeps in Section~\ref{sec:experiments}.

\section{The MARCO Model}
\label{sec:method}

Guided by our theoretical framing, MARCO introduces a minimal architectural modification to capture per-intent estimation gains while remaining strictly constrained in parameter and latency budgets. Rather than predicting a single scalar CVR, MARCO estimates per-intent CTR and CVR vectors and composes the total conversion probability $\Phi_{\mathrm{MARCO}}$ via the law of total probability as formulated in Eq.~\ref{eq:marco}. Estimating each $\mu_k$ on a homogeneous click sub-population eliminates the label conflation that causes systematic group-conditional bias. In our production platform, we instantiate MARCO at $K=2$.

\begin{figure}[t]
\centering
\includegraphics[width=0.5\columnwidth]{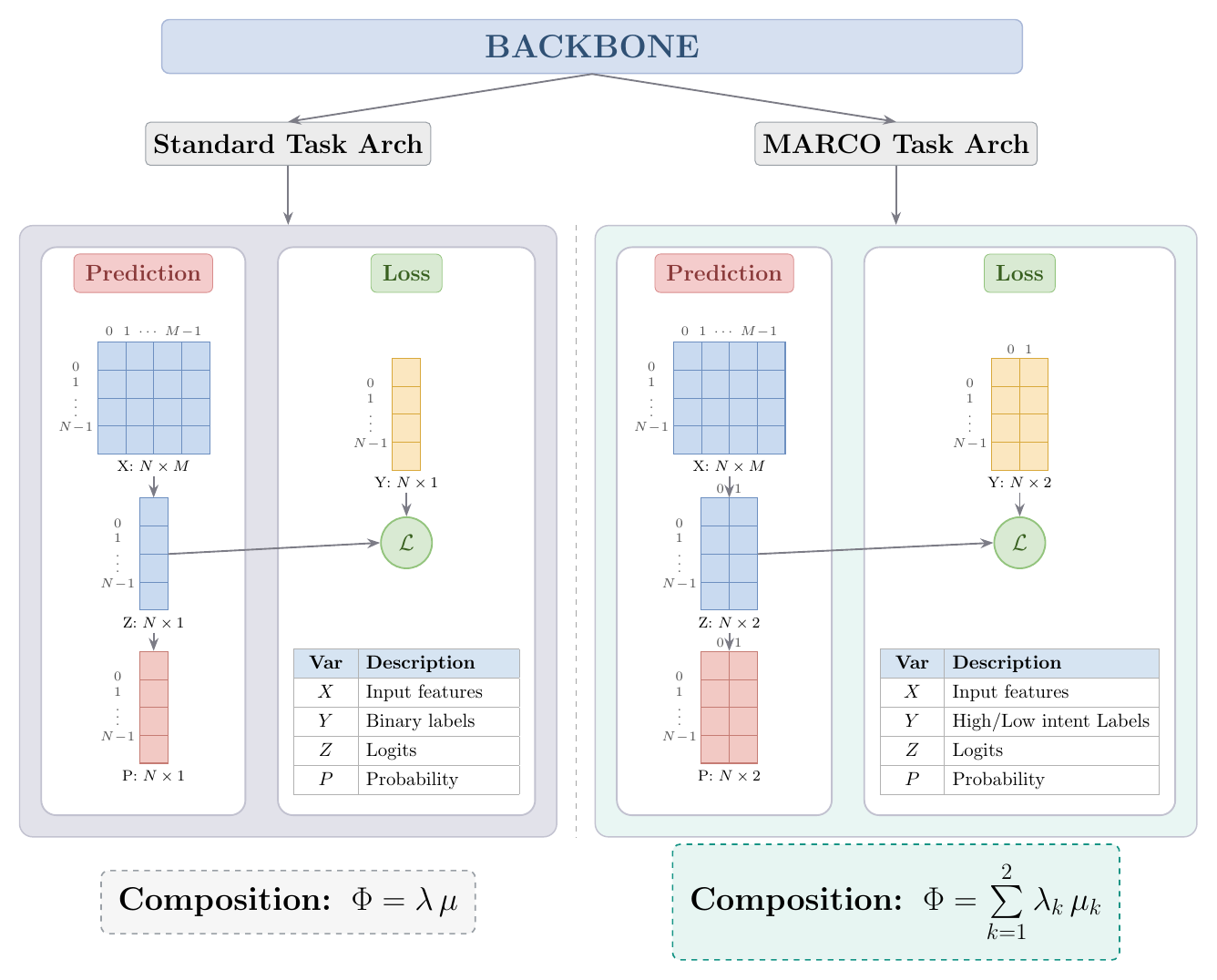}
\caption{Model architecture: Standard vs. MARCO.}
\label{fig:model_arch}
\end{figure}

\paragraph{Free Behavioral Supervision.}
MARCO leverages logged user interaction types as zero-cost supervision targets. Rather than requiring manual annotation, the intent label for each converted click is derived directly from physical UI behavior, serving as a behavioral pretext signal for latent intent. At training time, each per-intent CVR head $\mu_k$ is trained exclusively on clicks belonging to intent category $k$ (attributed via Section~\ref{sec:attribution}). At serving time, because the true click intent is unobserved prior to user interaction, MARCO evaluates all per-intent CVR heads and weights them by the predicted intent distribution $\hat{\pi}_k = \hat{\lambda}_k / \hat{\lambda}$, computing the expected CVR over predicted latent intents.

\paragraph{Vector-Head Architecture.}
As shown in Figure~\ref{fig:model_arch}, MARCO implements this decomposition via multi-output task heads without altering the underlying feature representations or model backbones. Shared embedding layers and deep neural network (DNN) representations remain intact for both CTR and CVR models. Alongside the baseline scalar projection, each model appends a $K$-dimensional vector head ($\mathbb{R}^d \to \mathbb{R}^K$, where $d$ denotes the final backbone representation dimension). The CTR head emits the complete per-intent click rate vector $(\lambda_1, \dots, \lambda_K)$, while the CVR head emits the corresponding conversion rate vector $(\mu_1, \dots, \mu_K)$. The original scalar output is retained strictly as an operational baseline for automated fallback (Section~\ref{sec:system_design}).

\paragraph{Parameter Efficiency \& Operational Overhead.}
Constructing $K$ separate dedicated models would significantly inflate parameter footprint, memory consumption, and serving latency, causing severe model fragmentation~\cite{luo2026lattice}. By expanding scalar projection heads into $K$-vector heads within existing backbones, MARCO adds fewer than $10^{-7}\%$ parameters relative to total model capacity. Because representation layers and embedding lookups are shared, serving latency remains entirely unaffected while successfully restoring per-intent calibration.

\paragraph{System Integration.}
MARCO operates within an end-to-end industrial ranking architecture (Figure~\ref{fig:system}), where data logging, credit attribution, model training, online serving, and real-time calibration are managed by independent infrastructure components. For the composite score in Eq.~\ref{eq:marco} to remain semantically coherent, all pipeline stages must maintain strict agreement on intent category assignments—a system requirement we formalize and enforce in Section~\ref{sec:attribution}.

\section{Attribution and System Design}
\label{sec:attribution}

Deploying the decomposed composition Eq.~\ref{eq:marco} in a distributed industrial ranking environment requires all logging, training, serving, and calibration pipelines to maintain strict agreement on which interaction defines the credited click intent for every impression. This section formalizes the credit attribution rationale (\S\ref{sec:credit_assignment}), establishes the cross-pipeline consistency invariants (\S\ref{sec:consistency}), and details the production infrastructure built to enforce them at scale (\S\ref{sec:system_design}).

\subsection{Credit Assignment Framework}
\label{sec:credit_assignment}

In conversion-optimized advertising, a user's pre-conversion journey spans multiple impressions and multi-touch interaction sequences. We model this trajectory as an ordered history $H$ of $n$ impressions, where each impression $\text{imp}_i$ logged at render time $\tau_i$ contains a multiset $C_i$ of $m_i \ge 0$ within-impression clicks logged at times $t_{i,j}$:
\begin{equation}
H = \big[ (\text{imp}_i, \tau_i, C_i) \big]_{i=1}^{n} \longrightarrow \text{conv}, \quad \text{where } C_i = \big\{ (\text{click}_{i,j}, t_{i,j}) \big\}_{j=1}^{m_i}
\end{equation}
with $\tau_1 < \tau_2 < \dots < \tau_n$ and $\tau_i \le t_{i,1} \le t_{i,2} \le \dots \le t_{i,m_i}$. Because a single rendered ad unit can generate a sequence of heterogeneous interactions (e.g., a social like at $t_{i,1}$ followed by a CTA tap at $t_{i,3}$), credit attribution must assign two non-negative weight vectors: a cross-impression weight $w_i$ and a within-impression click weight $v_{i,j}$. The credited conversion outcome is distributed across training tuples via the joint weight $w_i \cdot v_{i,j}$, yielding the per-impression training target decomposition:
\begin{equation}\label{eq:attribution_target}
\bar{\Phi}(H) = \sum_{i=1}^{n} w_i \sum_{j=1}^{m_i} v_{i,j} \cdot P(\text{conv} \mid \text{imp}_i, \text{click}_{i,j})
\end{equation}
enforcing $\sum_i w_i = 1$ and $\sum_j v_{i,j} = 1$ to guarantee total conversion mass.

\paragraph{Bias-Variance Tradeoff \& RL Analog.}
Selecting $(w_i, v_{i,j})$ represents a structural bias-variance tradeoff analogous to temporal-difference return estimation in reinforcement learning. As summarized in Table~\ref{tab:attribution_tradeoffs}, impressions and clicks serve distinct causal roles:
\begin{itemize}[leftmargin=*]
    \item \textbf{Impression Axis (Bias Control):} An impression is the external stimulus presented to the user. Weighting towards the \emph{last impression} ($w_n = 1$) isolates the causally proximal context that directly triggered conversion, minimizing context attribution bias (analogous to Monte Carlo return estimation).
    \item \textbf{Click Axis (Variance Control):} A click sequence within a fixed impression represents a stochastic exploration process. Selecting the \emph{first click} ($v_{i,1} = 1$) captures the user's immediate intent signature before subsequent taps inject behavioral noise, minimizing intent estimation variance (analogous to $\text{TD}(0)$ temporal-difference learning).
\end{itemize}

\begin{table}[t]
\caption{Attribution design space and its qualitative rationale: bias, variance, real-time feasibility, and RL credit-assignment analogs.}
\label{tab:attribution_tradeoffs}
\centering
\setlength{\tabcolsep}{3pt}
\begin{tabular}{l cc ccc l}
\toprule
& \textbf{Imp.} & \textbf{Click} & \textbf{Bias} & \textbf{Var.} & \textbf{Feasibility} & \textbf{RL analog} \\
\midrule
First-touch  & First & First & High & Low  & Deterministic & TD(0) \\
Last-touch   & Last  & Last  & Low  & High & Req.\ buffering & Monte Carlo \\
Time-decay   & Wtd.  & Wtd.  & Low  & Med. & Req.\ full hist. & GAE($\lambda$) \\
\textbf{MARCO} & \textbf{Last} & \textbf{First} & \textbf{Low} & \textbf{Low} & \textbf{Deterministic} & \textbf{--} \\
\bottomrule
\end{tabular}
\end{table}

\paragraph{Production Feasibility Constraints.}

MARCO's credit-assignment policy balances streaming feasibility on the click axis with context proximity on the impression axis. The consistency conditions of Section~\ref{sec:consistency} mandate that credited intent be deterministic and instantly resolvable without event buffering, ruling out last-click and time-decay rules in favor of first-click attribution. This choice carries minimal empirical risk: across the vast majority of impressions, initial and subsequent clicks share the same intent category, ensuring high intent alignment regardless of the click-axis rule. On the impression axis, we select the last impression to minimize context attribution bias, as it represents the causally proximal stimulus. The resulting \emph{last-impression, first-click} policy is both statistically robust and instantly resolvable at serving time.

\subsection{Cross-Pipeline Consistency Conditions}
\label{sec:consistency}
Because CTR and CVR training and calibration operate as independent distributed services, all four pipelines must resolve to the exact same credited click intent $\kappa(i) \in \{1, \dots, K\}$ for every impression $i \in I$. Training consistency requires
  $\kappa_{\mathrm{CTR}}^{\mathrm{train}}(i) = \kappa_{\mathrm{CVR}}^{\mathrm{train}}(i)$, calibration consistency requires $\kappa_{\mathrm{CTR}}^{\mathrm{cali}}(i) = \kappa_{\mathrm{CVR}}^{\mathrm{cali}}(i)$, and the end-to-end invariant requires all four to agree as  $\kappa_{\mathrm{CTR}}^{\mathrm{train}}(i) = \kappa_{\mathrm{CVR}}^{\mathrm{train}}(i) = \kappa_{\mathrm{CTR}}^{\mathrm{cali}}(i) = \kappa_{\mathrm{CVR}}^{\mathrm{cali}}(i)$.  The first two align the training and calibration stages internally, but they do not link training to calibration. The shared single-source dedup cache of Section~\ref{sec:system_design} supplies that link and enforces the invariant by construction.

\subsection{Production System Engineering}
\label{sec:system_design}

To enforce global consistency across distributed streaming environments (Figure~\ref{fig:system}), MARCO deploys two core infrastructure services: a \emph{Click Metadata Table} recording event-level payloads and a low-latency, durable key-value \emph{Deduplication Cache}.

\paragraph{Concurrent Click Attribution.}
When interaction events arrive concurrently across distributed edge servers, the engine executes an atomic \emph{first-write-wins} check on the Deduplication Cache. The earliest click event within the last attributed impression is registered as the credited intent $\kappa(i)$, while subsequent clicks are marked as deduplicated. This guarantees key uniqueness and write-atomicity under heavy write concurrency.

\paragraph{Schema-Agnostic Conversion Attribution.}
Offsite conversions arrive via asynchronous pipelines with heterogeneous user identifier schemas. When a conversion event triggers a lookup, the attribution system queries the Deduplication Cache using whichever identifier is present. Because all identifier paths resolve to the unified underlying cache entry established during click processing, all downstream consumers (CTR/CVR trainers and calibrators) observe the identical intent label, fulfilling consistency conditions of Section~\ref{sec:consistency}.

\paragraph{Composition-Level Fallback Guardrail.}
To protect auction delivery against model dilution or upstream prediction outages in concurrent A/B testing environments, MARCO incorporates a zero-latency fallback operator:
\begin{equation}\label{eq:fallback}
\Phi = \Phi_{\mathrm{MARCO}} + \mathbb{I}(\Phi_{\mathrm{MARCO}} = 0) \cdot \Phi_{\mathrm{std}}
\end{equation}
If experimental dilution causes intent-specific predictions to vanish (e.g., $\Phi_{\mathrm{MARCO}} = 0$), the system reverts to the standard scalar score $\Phi_{\mathrm{std}} = \lambda \mu$, guaranteeing no-worse-than-baseline robustness.

\paragraph{Cold-Start Protocol.}
MARCO addresses two operational cold-start scenarios:
\begin{enumerate}[leftmargin=*]
    \item \textbf{Unseen Click Types:} New UI interaction features default to the low-intent group, a low-risk strategy since high-intent actions represent established, high-conversion navigation paths.
    \item \textbf{Unwarmed Calibration Services:} To resolve the circular dependency where the real-time calibrator requires prediction traffic before MARCO goes live, we deploy a two-phase dummy composition: $\Phi_{\mathrm{phase1}} = \Phi_{\mathrm{std}} + 0 \cdot \Phi_{\mathrm{MARCO}}$.
    In Phase 1, scores remain identical to production baseline while intent-specific heads generate shadow predictions to warm up the calibration service. Once calibration variance converges (typically 2–3 days), Phase 2 activates the full MARCO composition (Eq.~\ref{eq:fallback}) with zero delivery perturbation.
\end{enumerate}

\section{Experiments}
\label{sec:experiments}
We evaluate MARCO across both offline log benchmarks and large-scale online A/B experiments on Meta's production advertising platforms. All experiments instantiate the framework at binary intent granularity ($K=2$), capturing the primary contrast based on conversion likelihood.
Online experiments span multi-week evaluation windows covering hundreds of millions of users with user-level randomization. All reported online gains are statistically significant at $p < 0.05$ under $95\%$ confidence intervals. Automated fallback (\S\ref{sec:system_design}) triggers on $< 0.1\%$ of traffic, ensuring near-complete treatment exposure.
Our empirical evaluation addresses four primary \textbf{R}esearch \textbf{Q}uestions:
\begin{itemize}[leftmargin=*]
    \item $\mathbf{RQ1}$ \textbf{Per-Intent Calibration:} Does the observability barrier of Proposition~\ref{prop:observability} hold in practice? Do progressively stronger single-output methods such as auxiliary tasks, historical features, and MoE gating remain unable to reduce per-intent miscalibration, while a model that exposes per-intent outputs, such as MARCO, eliminates it?
    \item $\mathbf{RQ2}$ \textbf{Composed Post-Impression NE:} Does MARCO improve end-to-end post-impression conversion prediction, and do empirical results align with theoretical headroom ($\Delta_{\mathcal{F}}$) and routing efficiency ($\hat{\eta}$)?
    \item $\mathbf{RQ3}$ \textbf{Pre-Impression Funnel Recall:} Does restoring per-intent calibration improve candidate retrieval fidelity in early ranking stages?
    \item $\mathbf{RQ4}$ \textbf{Online Business Impact:} What is the topline business impact in live auctions, and how does score separation alter candidate ranking?
\end{itemize}

\subsection{Offline Evaluation}
\label{sec:offline}

\subsubsection{Post-click: per-intent calibration $\mathbf{(RQ1)}$}
\label{sec:exp_post_click}


We compare MARCO against four progressively stronger paradigms for utilizing intent signals short of structural decomposition:
\begin{enumerate}[leftmargin=*]
    \item \textbf{Baseline (Pooled CVR):} An ESMM-style vertical funnel model~\cite{ma2018esmm,wen2020esm2} fitting a single conversion head over all clicks.
    \item \textbf{+ Auxiliary Intent Tasks:} Appends auxiliary task heads to predict per-intent distribution, testing whether auxiliary representation learning resolves pooled head miscalibration.
    \item \textbf{+ Historical Intent Feature:} Feeds historical intent distributions as impression-time input features—the strongest intent signal observable prior to interaction (Proposition~\ref{prop:observability}).
    \item \textbf{+ MoE on Historical Intent:} Deploys a Mixture-of-Experts (MoE) architecture gated on historical intent features, testing whether architectural gating eliminates subgroup bias.
\end{enumerate}

\begin{table}[th]
\caption{Per-intent calibration-error reduction relative to the pooled baseline (\%), higher is better. Single-output architectures are bounded by Proposition~\ref{prop:observability}, whereas MARCO eliminates over $99\%$ of per-intent error.}
\label{tab:offline}
\centering
\begin{tabular}{lcc}
\toprule
Model & High intent & Low intent \\
\midrule
Baseline (pooled) & $0\%$ & $0\%$ \\
+ Auxiliary intent task & $0\%$ & $0\%$ \\
+ Historical intent feature & $1\%$ & $3\%$ \\
+ MoE on historical intent & $2\%$ & $8\%$ \\
\textbf{MARCO (decomposition)} & $\mathbf{> 99\%}$ & $\mathbf{> 99\%}$ \\
\bottomrule
\end{tabular}
\end{table}

Table~\ref{tab:offline} presents the evaluation ladder. Extra capacity through the auxiliary task leaves per-intent calibration unchanged ($0\%$). Historical intent features and MoE gating yield marginal improvements ($\le 8\%$). Because these baselines emit a single output $\mu$ for a mixed click population, Proposition~\ref{prop:observability} imposes a strict observational upper bound: a single prediction cannot simultaneously match distinct conditional targets $\mu_k$. 
MARCO removes essentially all of the per-intent calibration error, more than an order of magnitude above the nearest alternative, so the effect is unambiguous.

This evaluation ladder focuses on single-output architectures to empirically validate the observability limit in Proposition~\ref{prop:observability}. Because these models emit a single prediction prior to interaction, their calibration ceiling remains fixed regardless of added tasks, historical features, or MoE gating. While any architecture with per-intent output heads can bypass this barrier, MARCO provides the minimal, zero-latency design needed at industrial scale.

\subsubsection{Composed Post-impression NE $\mathbf{(RQ2)}$}
\label{sec:exp_post_imp}



The post-impression evaluation assesses the end-to-end conversion prediction $\Phi$ across the global impression space, composing CTR and CVR scores via Eq.~\ref{eq:standard} for the baseline and Eq.~\ref{eq:marco} for MARCO. We evaluate offline performance on production data~\cite{ipcf} across a 7-day evaluation window. Table~\ref{tab:postimp} reports post-impression Normalized Entropy gains across a comprehensive evaluation matrix, systematically varying CTR backbone capacity (Base, Large 5$\times$, and an Oracle router evaluating intent routing under ground-truth $P(k \mid \text{click})$ while using predicted click rates), composition architecture (MARCO vs CTR-split-only), and calibration adjustment strategies (None, CVR-only, CTR-only, and joint CTR \& CVR).

\begin{table*}[t]
\centering
\caption{Composed-model post-impression NE gain across evaluation settings (\%), higher is better, with 95\% confidence intervals reported against standard composition.}
\label{tab:postimp}
\small
\setlength{\tabcolsep}{14pt}
\renewcommand{\arraystretch}{1.18}

\begin{tabular}{c l l r@{\hspace{3pt}}l}
\toprule
\textbf{Related RQ} & \textbf{Composition Logic} & \textbf{Calibration Adjustment} & \multicolumn{2}{c}{\textbf{NE Gain (\%) [95\% CI]}} \\
\midrule
\multicolumn{5}{l}{\textbf{Ablation Analysis (Base Model)}} \\
\midrule
2.1      & MARCO          & None       & $+0.054$ & $[+0.043, +0.067]$ \\
2.1, 2.2 & MARCO          & CTR \& CVR & $+0.223$ & $[+0.210, +0.235]$ \\
2.2.1    & MARCO          & CVR only   & $+0.181$ & $[+0.168, +0.194]$ \\
2.2.2    & MARCO          & CTR only   & $+0.097$ & $[+0.088, +0.110]$ \\
2.3      & CTR-split-only & None       & $+0.002$ & $[-0.020, +0.007]$ \\
\midrule
\multicolumn{5}{l}{\textbf{Capacity Scaling \& Theoretical Bound}} \\
\midrule
2.4      & MARCO (Large CTR, $5\times$) & None & $+0.092$ & $[+0.067, +0.112]$ \\
2.4      & MARCO (Oracle Router)        & None & $+2.268$ & $[+2.235, +2.302]$ \\
\bottomrule
\end{tabular}
\end{table*}

Table~\ref{tab:postimp} reveals four key findings directly validating our theoretical framework:
\begin{enumerate}[leftmargin=*]
    \item \textbf{$\mathbf{RQ2.1}$ Empirical Weak Dominance:} MARCO strictly outperforms the pooled baseline in post-impression NE across all model capacities and calibration settings in Table~\ref{tab:postimp} (e.g., $+0.054\%$ pre-calibration under Base CTR). This universal, strictly positive lift empirically validates Weak Dominance ($\Delta_{\mathcal{F}} \ge 0$, Theorem~\ref{app:respgate}(i)), confirming that structural intent decomposition expands the hypothesis class without inflating population risk.

    \item \textbf{$\mathbf{RQ2.2}$ Calibration Service Synergy:} Pre-calibration gains jump by over $4\times$ when passing through the online calibration service ($+0.054\%$ to $+0.223\%$ for Base CTR under joint CTR \& CVR calibration). Isolating individual calibration components reveals that CVR-only calibration accounts for the majority of this lift ($+0.181\%$), whereas CTR-only calibration contributes $+0.097\%$. Unmixing per-intent outputs allows CVR calibrators to adjust each intent head against its homogeneous sub-population, resolving group-conditional errors that aggregate calibrators are structurally blind to (Proposition~\ref{prop:observability}).

    \item \textbf{$\mathbf{RQ2.3}$ Estimation Coupling Bias ($\varepsilon_{\mathrm{couple}}$):} The CTR-split-only ablation yields a statistically negligible pre-calibration gain ($+0.002\%$ $[-0.020\%, +0.007\%]$). Because $\sum_k \lambda_k = \lambda$ (Lemma~\ref{prop:bridge}), splitting the CTR head without CVR-side decomposition provides zero CVR headroom while introducing estimation coupling bias $\varepsilon_{\mathrm{couple}}$. This confirms that CVR-side structural intent decomposition is strictly required to capture post-impression gains.
    
    \item \textbf{$\mathbf{RQ2.4}$ Routing Efficiency Expansion ($\hat{\eta}$):} Expanding CTR model capacity from $1\times$ (Base) to $5\times$ (Large) increases the pre-calibration composed NE gain from $+0.054\%$ to $+0.092\%$ in Table~\ref{tab:postimp}, demonstrating that scaling intent prediction capacity directly unlocks higher post-impression accuracy. When intent routing is driven by an Oracle router ($\widehat{\pi} = \pi$), the NE gain reaches $+2.268\%$ $[+2.235\%, +2.302\%]$, establishing the empirical upper bound for realizable theoretical headroom $\Delta_{\mathcal{F}}$. Benchmarking pre-calibration gains against Oracle headroom ($\Delta_{\mathcal{F}} = +2.268\%$) confirms that scaling CTR capacity increases realized routing efficiency $\hat{\eta}$ from $2.38\%$ (Base) to $4.06\%$ (Large). This capacity progression directly validates Proposition~\ref{prop:routing_monotonicity}: expanding CTR capacity enlarges the realizable routing class $\Pi^{(c)}$, systematically reducing the heterogeneity-weighted routing cost $\varepsilon_{\mathrm{route}}$ and raising serving efficiency $\hat{\eta}$ toward $100\%$.
\end{enumerate}

\subsubsection{Pre-Impression Funnel Recall $\mathbf{(RQ3)}$}
\label{sec:exp_pre_imp}



The pre-impression stage governs early-stage candidate retrieval, deciding which ad candidates survive to full auction scoring. We evaluate candidate retrieval fidelity on offline production logs across a 3-day evaluation window by truncating candidate sets at an early-funnel scale of $\mathcal{O}(10^2)$. Retrieval quality is measured using value-weighted recall (the ratio of value delivered by the truncated candidate list to the ideal unconstrained list). MARCO achieves a $+0.38\%$ lift in value-weighted candidate retrieval recall over the pooled baseline. Restoring per-intent score calibration at inference improves rank-order fidelity early in the funnel, ensuring higher-quality candidates advance to final auction scoring.

\subsection{Online Business Impact $\mathbf{(RQ4)}$}
\label{sec:online}



MARCO has been incrementally deployed across multiple ad surfaces, conversion types, and traffic segments on Meta's social media platforms.

\subsubsection{Cumulative Production Impact}
\label{sec:cumulative}

MARCO was deployed through several successive production launches. Each launch was conducted as an independent experiment on a disjoint traffic segment, measured against its own concurrent user-randomized holdback. Every launch yielded positive gains, collectively covering a substantial fraction of Meta's ad traffic. Evaluated globally across all combined launch segments, MARCO delivered a $+0.98\%$ cumulative lift in topline metrics ($95\%$ CI: $[+0.86\%, +1.12\%]$). The confidence interval is a bootstrap percentile interval computed across launches, reflecting performance stability across distinct rollout segments.

\subsubsection{Detailed Holdback Analysis}
\label{sec:holdback}


To evaluate MARCO's production performance under live auction dynamics, we conducted a multi-week, user-randomized holdback experiment on a primary advertising surface, comparing the proposed intent-decomposed composition against the standard single-click baseline. MARCO achieved a $+2.80\%$ lift in conversions per click with a $95\%$ confidence interval of $[+2.50\%, +3.10\%]$, confirming substantial end-to-end performance gains in live production traffic.


The online conversion rate improvement ($+2.80\%$) significantly exceeds the offline impression-space NE gain ($+0.223\%$). This divergence is expected and attributable to two fundamental structural mechanisms:
\begin{enumerate}[leftmargin=*]
    \item \textbf{Probability Space Compression:} Offline NE evaluates the global impression space, where the served score $\Phi = \lambda\mu$ scales on-click conversion rates $\mu$ by the baseline click-through rate $\lambda$. Under squared-loss score headroom analysis (Proposition~\ref{prop:bridge-sq}), impression-level risk gaps scale proportionally to $\lambda^2$. Because click-through rates are small ($\lambda \ll 1$), offline impression-space metrics undergo geometric compression relative to raw on-click conversion gains.
    \item \textbf{Auction Reallocation Dynamics:} Modern real-time ad auctions operate as competitive winner-take-all mechanisms. By eliminating group-conditional miscalibration, MARCO achieves superior score separation across intent strata (Section~\ref{sec:distribution}). This enhanced score fidelity allows high-intent candidate ads to systematically win auctions while suppressing low-converting traffic, dynamically reallocating delivery budget toward higher-converting opportunities and amplifying calibrated scoring gains into non-linear online business lifts.
\end{enumerate}

\subsubsection{Prediction Distribution Analysis}
\label{sec:distribution}



\begin{figure}[t]
\centering
\includegraphics[width=0.5\columnwidth]{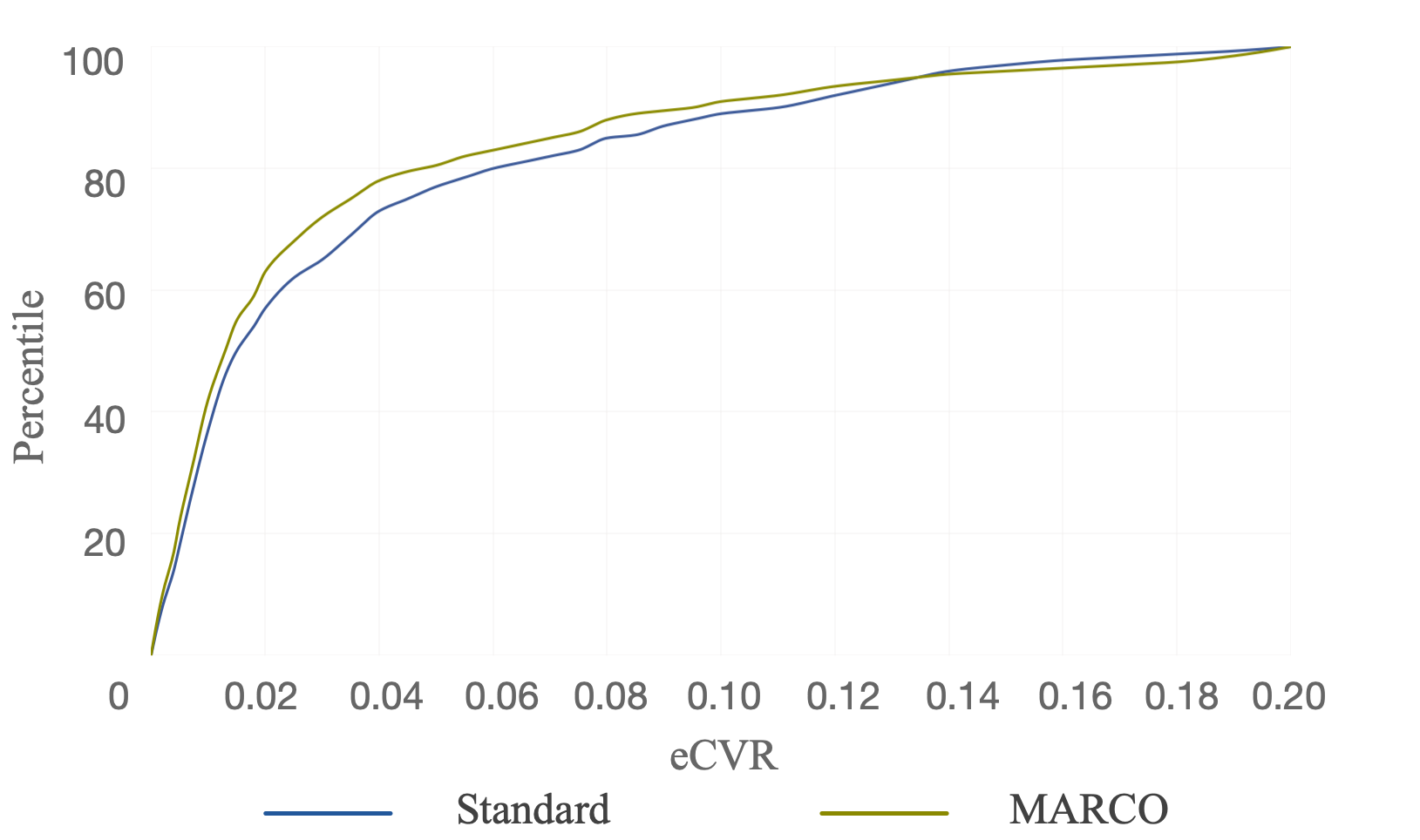}
\caption{Cumulative Distribution Function (CDF) of predicted conversion rates (eCVR) at the pre-impression stage, illustrating score separation under MARCO.}
\label{fig:ecvr}
\end{figure}

Figure~\ref{fig:ecvr} illustrates the mechanism driving online auction gains. For impressions prone to low-intent clicks, MARCO depresses expected conversion scores, causing the treatment CDF to rise faster in the low-eCVR regime ($< 0.04$). Conversely, for impressions prone to high-intent CTA taps, MARCO amplifies scores, shifting the treatment CDF below the baseline in the high-eCVR regime ($> 0.08$). This score separation provides the auction mechanism with high-fidelity conversion signals, systematically suppressing low-converting candidates and winning auctions for high-intent traffic.

\section{Discussion and Conclusion}
\label{sec:discussion}

In this section, we address privacy and ethical considerations, demonstrate the generalizability of intent decomposition to adjacent domains, and outline future technical directions.

\paragraph{Ethical Considerations and Advertiser Alignment.}
MARCO's design adheres strictly to platform privacy standards and ethical AI practices. First, regarding \textbf{data privacy}, MARCO relies exclusively on physical UI interactions already logged in standard interaction pipelines, requiring no new sensitive user profiling or off-platform tracking. Second, to ensure \textbf{statistical fairness across advertisers}, restoring subgroup calibration corrects systematic under-prediction on high-intent traffic and over-prediction on low-intent traffic, eliminating unfair auction penalties for advertisers optimizing for high-converting direct actions. Finally, MARCO remains \textbf{objective-agnostic}, dynamically routing candidate value based on advertiser-defined campaign goals rather than hard-coding a preference for direct sales over brand engagement.

\paragraph{Broad Generalizability and Diagnostic Framework.}
The principle of structural decomposition extends beyond click-intent prediction. Whenever a single supervised target aggregates events with structurally heterogeneous downstream success rates, single-output models inherit subgroup bias masked by aggregate calibration. Representative applications include app install campaigns (pooling organic-attributed and ad-driven installs with distinct 7-day retention rates), e-commerce funnels (merging quick-add interactions from product grids with considered-adds from detailed review pages), and video recommendation platforms (combining lightweight preview plays with intentional full-screen watch sessions).

Industrial practitioners can evaluate whether a target domain warrants intent decomposition prior to model engineering. By evaluating the impression-weighted Kullback-Leibler (KL) divergence spread across candidate interaction strata (following Lemma~\ref{app:excess-ce} and Lemma~\ref{lem:routing_cost}), one can quantify the exact theoretical headroom $\Delta_{\mathcal{F}}$ opened by decomposition directly from historical interaction logs.

\paragraph{Limitations and Future Directions.}
While MARCO achieves substantial production impact at binary intent granularity ($K=2$), theoretical headroom $\Delta_{\mathcal{F}}$ increases monotonically with finer intent partitions. We outline three key vectors for future research:

\begin{enumerate}[leftmargin=*]
    \item \textbf{Taxonomy Scaling (Depth \& Breadth):} Expanding $K$ can proceed along two dimensions: \emph{depth} (partitioning CTA clicks into sub-navigation types) and \emph{breadth} (incorporating post-impression interaction signals such as dwell time, video watch percentage, and organic engagement). Because Weak Dominance holds for arbitrary partitions (Theorem~\ref{app:respgate}(i)), finer taxonomies never degrade population risk.
    \item \textbf{Routing Efficiency Leverage ($\hat{\eta}$):} Realized online gain is governed by routing efficiency $\hat{\eta} = 1 - \varepsilon_{\mathrm{route}} / \Delta_{\mathcal{F}}$. Because routing error concentrates where per-intent conversion rates diverge (Lemma~\ref{lem:routing_cost}), scaling the capacity of the \emph{CTR intent predictor}—rather than the CVR model—is the highest-leverage vector to close the remaining headroom gap.
    \item \textbf{Sparse-Stratum Calibration Infrastructure:} As $K$ increases, interaction volume is divided across finer intent buckets. Rare interaction categories risk calibration variance due to sample sparsity. Developing dynamic calibrators capable of robust real-time convergence under long-tailed intent distributions remains an important system engineering challenge.
\end{enumerate}

Conventional ads ranking architectures conflate distinct user behaviors into a single click label, imposing a structural observability barrier that no model capacity can resolve. MARCO demonstrates that decoupling predictions by intent via zero-cost behavioral supervision restores subgroup calibration, resolves auction score distortions, and delivers substantial topline growth at industrial scale.

\bibliographystyle{unsrt}
\bibliography{references}

\begin{thebibliography}{10}

\bibitem{ma2018esmm}
Xiao Ma, Liqin Zhao, Guan Huang, Zhi Wang, Zelin Hu, Xiaoqiang Zhu, and Kun
  Gai.
\newblock Entire space multi-task model: An effective approach for estimating
  post-click conversion rate.
\newblock In {\em The 41st International ACM SIGIR Conference on Research \&
  Development in Information Retrieval}, SIGIR ’18, pages 1137--1140. ACM,
  2018.

\bibitem{wen2020esm2}
Hong Wen, Jing Zhang, Yuan Wang, Fuyu Lv, Wentian Bao, Quan Lin, and Keping
  Yang.
\newblock Entire space multi-task modeling via post-click behavior
  decomposition for conversion rate prediction.
\newblock In {\em Proceedings of the 43rd International ACM SIGIR Conference on
  Research and Development in Information Retrieval}, SIGIR ’20, pages
  2377--2386. ACM, 2020.

\bibitem{holtsmith1979poststrat}
D.~Holt and T.~M.~F. Smith.
\newblock Post stratification.
\newblock {\em Journal of the Royal Statistical Society: Series A},
  142(1):33--46, 1979.

\bibitem{little1993poststrat}
Roderick J.~A. Little.
\newblock Post-stratification: A modeler's perspective.
\newblock {\em Journal of the American Statistical Association},
  88(423):1001--1012, 1993.

\bibitem{wang2022escm2}
Hao Wang, Tai-Wei Chang, Tianqiao Liu, Jianmin Huang, Zhichao Chen, Chao Yu,
  Ruopeng Li, and Wei Chu.
\newblock Escm2: Entire space counterfactual multi-task model for post-click
  conversion rate estimation.
\newblock In {\em Proceedings of the 45th International ACM SIGIR Conference on
  Research and Development in Information Retrieval}, SIGIR ’22, pages
  363--372. ACM, 2022.

\bibitem{dcmt2023}
Feng Zhu, Mingjie Zhong, Xinxing Yang, Longfei Li, Lu~Yu, Tiehua Zhang, Jun
  Zhou, Chaochao Chen, Fei Wu, Guanfeng Liu, and Yan Wang.
\newblock Dcmt: A direct entire-space causal multi-task framework for
  post-click conversion estimation.
\newblock In {\em 2023 IEEE 39th International Conference on Data Engineering
  (ICDE)}, pages 3113--3125. IEEE, 2023.

\bibitem{xi2021aitm}
Dongbo Xi, Zhen Chen, Peng Yan, Yinger Zhang, Yongchun Zhu, Fuzhen Zhuang, and
  Yu~Chen.
\newblock Modeling the sequential dependence among audience multi-step
  conversions with multi-task learning in targeted display advertising.
\newblock In {\em Proceedings of the 27th ACM SIGKDD Conference on Knowledge
  Discovery \& Data Mining}, KDD ’21, pages 3745--3755. ACM, 2021.

\bibitem{ma2018mmoe}
Jiaqi Ma, Zhe Zhao, Xinyang Yi, Jilin Chen, Lichan Hong, and Ed~H. Chi.
\newblock Modeling task relationships in multi-task learning with multi-gate
  mixture-of-experts.
\newblock In {\em Proceedings of the 24th ACM SIGKDD International Conference
  on Knowledge Discovery \& Data Mining}, KDD ’18, pages 1930--1939. ACM,
  2018.

\bibitem{tang2020ple}
Hongyan Tang, Junning Liu, Ming Zhao, and Xudong Gong.
\newblock Progressive layered extraction (ple): A novel multi-task learning
  (mtl) model for personalized recommendations.
\newblock In {\em Fourteenth ACM Conference on Recommender Systems}, RecSys
  ’20, pages 269--278. ACM, 2020.

\bibitem{chang2023pepnet}
Jianxin Chang, Chenbin Zhang, Yiqun Hui, Dewei Leng, Yanan Niu, Yang Song, and
  Kun Gai.
\newblock Pepnet: Parameter and embedding personalized network for infusing
  with personalized prior information.
\newblock In {\em Proceedings of the 29th ACM SIGKDD Conference on Knowledge
  Discovery and Data Mining}, KDD ’23, pages 3795--3804. ACM, 2023.

\bibitem{platt1999probabilistic}
John~C. Platt.
\newblock Probabilistic outputs for support vector machines and comparisons to
  regularized likelihood methods.
\newblock In Alexander~J. Smola, Peter Bartlett, Bernhard Sch{\"o}lkopf, and
  Dale Schuurmans, editors, {\em Advances in Large Margin Classifiers}, pages
  61--74. MIT Press, 1999.

\bibitem{zadrozny2002transforming}
Bianca Zadrozny and Charles Elkan.
\newblock Transforming classifier scores into accurate multiclass probability
  estimates.
\newblock In {\em Proceedings of the eighth ACM SIGKDD international conference
  on Knowledge discovery and data mining}, KDD02, pages 694--699. ACM, 2002.

\bibitem{calmatters2023}
Yewen Fan, Nian Si, and Kun Zhang.
\newblock Calibration matters: Tackling maximization bias in large-scale
  advertising recommendation systems.
\newblock In {\em The Eleventh International Conference on Learning
  Representations}, 2023.

\bibitem{beyondpartitions2023}
Jia-Qi Yang, De-Chuan Zhan, and Le~Gan.
\newblock Beyond probability partitions: Calibrating neural networks with
  semantic aware grouping.
\newblock In {\em Advances in Neural Information Processing Systems 36},
  NeurIPS 2023, pages 58448--58460. Neural Information Processing Systems
  Foundation, Inc. (NeurIPS), 2023.

\bibitem{netflix2024intent}
Sejoon Oh, Moumita Bhattacharya, Yesu Feng, and Sudarshan Lamkhede.
\newblock {IntentRec}: Predicting user session intent with hierarchical
  multi-task learning.
\newblock {\em arXiv preprint arXiv:2408.05353}, 2024.

\bibitem{luo2026lattice}
Liang Luo, Yuxin Chen, Zhengyu Zhang, Mengyue Hang, Andrew Gu, Buyun Zhang,
  Boyang Liu, Chen Chen, Fan Yang, Feifan Gu, Huayu Li, Jade Nie, Jiayi Xu,
  Jiyan Yang, Jongsoo Park, Laming Chen, Longhao Jin, Qin Huang, Shali Jiang,
  Shiwen Shen, Shuaiwen Wang, Siyang Yuan, Tongyi Tang, Weilin Zhang, Xi~Liu,
  Xiaohan Wei, Yuchen Hao, Xiaozhen Xia, Yasmine Badr, Zeliang Chen, Chengze
  Fan, Dong Liang, Qianru Li, Sihan Zeng, Wenjun Wang, Yunlong He, Yinbin Ma,
  Maxim Naumov, Yantao Yao, Wenlin Chen, and Ellie~Dingqiao Wen.
\newblock Meta lattice: Model space redesign for cost-effective industry-scale
  ads recommendations.
\newblock In {\em Proceedings of the 32nd ACM SIGKDD Conference on Knowledge
  Discovery and Data Mining V.1}, KDD '26, pages 2335--2346. Association for
  Computing Machinery, 2026.

\bibitem{ipcf}
Mahanth~Kumar Beeraka, Chen Chen, Yining Lu, Briac Marcatte, Weikun Lyu, Brooke
  Bian, Enriko Aryanto, Ellie Wen, Mohamed~A. Radwan, Tianshan Cui, Wenjing Lu,
  Mohsen Malmir, and Yang Li.
\newblock Closing the online-offline gap: A scalable framework for composed
  model evaluation.
\newblock In {\em Proceedings of the Nineteenth ACM Conference on Recommender
  Systems}, RecSys ’25, pages 923--926. ACM, 2025.

\bibitem{luenberger1969}
David~G Luenberger.
\newblock {\em Optimization by vector space methods}.
\newblock John Wiley \& Sons, Nashville, TN, January 1997.

\bibitem{banerjee2005bregman}
A.~Banerjee, X.~Guo, and H.~Wang.
\newblock On the optimality of conditional expectation as a bregman predictor.
\newblock {\em IEEE Transactions on Information Theory}, 51(7):2664–2669,
  2005.

\bibitem{gneiting2007scoring}
Tilmann Gneiting and Adrian~E Raftery.
\newblock Strictly proper scoring rules, prediction, and estimation.
\newblock {\em Journal of the American Statistical Association},
  102(477):359–378, March 2007.

\bibitem{dawid1982calibration}
A.~Philip Dawid.
\newblock The well-calibrated bayesian.
\newblock {\em Journal of the American Statistical Association},
  77(379):605--610, 1982.

\bibitem{mityagin2015zero}
Boris Mityagin.
\newblock The zero set of a real analytic function.
\newblock {\em arXiv preprint arXiv:1512.07276}, 2015.

\bibitem{minka2000bma}
Thomas~P. Minka.
\newblock Bayesian model averaging is not model combination.
\newblock Technical report, MIT Media Lab, 2000.
\newblock Technical note.

\bibitem{monteith2011bmc}
Kristine Monteith, James~L. Carroll, Kevin Seppi, and Tony Martinez.
\newblock Turning {B}ayesian model averaging into {B}ayesian model combination.
\newblock In {\em The 2011 International Joint Conference on Neural Networks
  (IJCNN)}, pages 2657--2663, 2011.

\bibitem{hornik1989universal}
Kurt Hornik, Maxwell Stinchcombe, and Halbert White.
\newblock Multilayer feedforward networks are universal approximators.
\newblock {\em Neural Networks}, 2(5):359–366, January 1989.

\end{thebibliography}

\pagebreak
\appendix
\section{Derivation of the Headroom and Weak Dominance}\label{sec:appendix}

This appendix gives the full statements and proofs behind the theory section. We prove the observability barrier, the Bayes equivalence, the excess-risk identity, weak dominance in any function class, the exact squared-loss headroom, the genericity of strict improvement, the score bridge, and the signed decomposition of the realized gain. The squared-loss results carry a closed form. The normalized-entropy results keep only the loss-free parts, for the reason given in the mixture-of-experts lemma (Lemma~\ref{prop:moe}).

\subsection{Setup and notation}

We condition on a click $C=1$ unless the impression level is named. Let $X$ be impression-time features with impression distribution $\rho$. Let $k\in\{1,\dots,K\}$ index the click intent over $K$ disjoint categories, observed at training from the logged click type. Let $Y\in\{0,1\}$ be the conversion. Define
\begin{equation}
\begin{aligned}
\mu_k&=\mathbb{E}[Y\mid X{=}x,K{=}k],\quad
\pi_k=P(K{=}k\mid X{=}x,C{=}1),\\
\mu&=\mathbb{E}[Y\mid X{=}x,C{=}1].
\end{aligned}
\end{equation}
The law of total probability gives the exact identity $\mu=\sum_{k=1}^K \pi_k\mu_k$. This identity assumes no homogeneity of clicks. Let $\lambda=P(C{=}1\mid X{=}x)$ be the click-through rate and $\lambda_k=\lambda\pi_k$ the per-intent click rate, so $\sum_k\lambda_k=\lambda$ and $\pi_k=\lambda_k/\lambda$.

Let $\mathcal{F}$ be the hypothesis class of a single pooled conversion head. The intent-decomposed class is
\begin{equation}\label{app:eq:Gclass}
\mathcal{G}=\Big\{\, g=\sum_{k=1}^K \pi_k f_k:\ f_k\in\mathcal{F}\,\Big\}.
\end{equation}
MARCO serves the impression-level score $\Phi_{\mathrm{MARCO}}=\sum_k \lambda_k\mu_k$. The standard model serves $\Phi_{\mathrm{std}}=\lambda\mu$.

\subsection{Bayes equivalence}

\begin{lemma}[Bayes equivalence]\label{prop:bayes_equivalence}
At the true conditionals $\Phi_{\mathrm{MARCO}}=\Phi_{\mathrm{std}}$.
\end{lemma}
\begin{proof}
Use $\lambda_k=\lambda\pi_k$ and $\sum_k\pi_k=1$. Then $\Phi_{\mathrm{MARCO}}=\sum_k\lambda\pi_k\mu_k=\lambda\sum_k\pi_k\mu_k=\lambda\mu=\Phi_{\mathrm{std}}$.
\end{proof}

At the population level the two scores coincide. The Bayes-optimal impression score is unchanged by the decomposition. Any deployed gain is therefore a finite-capacity estimation and calibration effect. The click-rate decomposition adds no independent factor. It supplies the routing $\pi_k=\lambda_k/\lambda$ and the total $\lambda=\sum_k\lambda_k$. The served conversion factor is $\sum_k\pi_k\mu_k$, an element of $\mathcal{G}$ whose routing is produced by the click-rate head. The analysis reduces to the single conversion head of \eqref{app:eq:Gclass}.

\subsection{Excess risk}

We first record the squared-loss identity, then the cross-entropy identity.

Squared loss uses the risk $\mathcal{R}(f)=\mathbb{E}[(Y-f)^2]$. Write $\langle\cdot,\cdot\rangle$ and $\lVert\cdot\rVert$ for the $L^2(\rho)$ inner product and norm. he identity below is the standard $L^2(\rho)$ projection and bias-variance decomposition~\cite{luenberger1969}.

\begin{lemma}[Squared-loss excess risk]\label{app:excess-sq}
For any $f$, $\mathcal{R}(f)=\mathcal{R}(\mu)+\lVert f-\mu\rVert^2$. The infimum of $\mathcal{R}$ over a closed set is attained at the $L^2(\rho)$ projection of $\mu$ onto that set.
\end{lemma}
\begin{proof}
Expand
\begin{equation}
\begin{aligned}
\mathcal{R}(f)={}&\mathbb{E}[(Y-\mu)^2]\\
&+2\,\mathbb{E}[(Y-\mu)(\mu-f)]\\
&+\mathbb{E}[(\mu-f)^2].
\end{aligned}
\end{equation}
The cross term vanishes because $\mathbb{E}[Y-\mu\mid X]=0$. The first term is $\mathcal{R}(\mu)$. The last term is $\lVert f-\mu\rVert^2$. Minimizing $\lVert f-\mu\rVert^2$ over a closed set is projection.
\end{proof}

Cross-entropy uses the risk $\mathcal{R}(f)=\mathbb{E}[-Y\log f-(1-Y)\log(1-f)]$ for $f$ valued in $(0,1)$. Normalized entropy is $\mathrm{NE}(f)=\mathcal{R}(f)/H(\bar y)$ with $\bar y=\mathbb{E}[Y]$ and $H$ the binary entropy. $H(\bar y)$ is a fixed positive constant, so NE and cross-entropy share minimizers and orderings. Every statement below applies to both.

\begin{lemma}[Cross-entropy excess-risk identity]\label{app:excess-ce}
For any $(0,1)$-valued $f$,
\begin{equation}
\mathcal{R}(f)=\mathcal{R}(\mu)+\mathbb{E}_X\big[\mathrm{KL}(\mathrm{Bern}(\mu)\,\Vert\,\mathrm{Bern}(f))\big].
\end{equation}
Hence $\mathcal{R}(f)\ge\mathcal{R}(\mu)$ with equality iff $f=\mu$ almost everywhere.
\end{lemma}
\begin{proof}
Cross-entropy is the Bregman divergence generated by $\psi(p)=p\log p+(1-p)\log(1-p)$~\cite{banerjee2005bregman,gneiting2007scoring}. Write $D_\psi(y,f)=\psi(y)-\psi(f)-\psi'(f)(y-f)$. Then $\mathbb{E}[D_\psi(Y,f)\mid X]=\mathbb{E}[\psi(Y)\mid X]-\psi(f)-\psi'(f)(\mu-f)$ because $\mathbb{E}[Y\mid X]=\mu$. Subtracting the same expression at $f=\mu$ leaves $\mathrm{KL}(\mathrm{Bern}(\mu)\,\Vert\,\mathrm{Bern}(f))$. Take the expectation over $X$. Nonnegativity and the equality case follow from strict convexity of $\psi$.
\end{proof}

The excess risk under cross-entropy is an averaged Kullback-Leibler divergence between the true per-impression conversion rate and the prediction. Under the classical calibration--refinement decomposition of a proper scoring rule~\cite{dawid1982calibration}, this gap splits into a calibration term, which measures how far predictions sit from the conditional conversion rate they are grouped with, and a refinement term, which measures the residual heterogeneity within those groups. A per-intent calibration step drives the calibration term to zero within each intent, while decomposition attacks the refinement term by splitting the pooled group into homogeneous intent slices. The log-loss gap and per-intent calibration are thus the same object measured on each intent slice.

\subsection{Observability barrier}

\begin{proof}[Proof of Proposition~\ref{prop:observability}]
Fix $x$ and condition on a click. Being calibrated on intent category $k$ means the emitted value equals that category's conditional mean, $\mu=\mu_k$. Simultaneous calibration on all $K$ categories thus requires $\mu=\mu_k$ for every $k$, which is impossible whenever the $\mu_k$ are not all equal, since a single scalar equals at most one element of a set of distinct reals; the same fact gives the at-most-one claim. Under a strictly proper scoring rule the risk-minimizing single output is the pooled conditional mean $\mathbb{E}[Y\mid X{=}x,C{=}1]=\sum_k\pi_k\mu_k$ (Lemmas~\ref{app:excess-sq} and~\ref{app:excess-ce}), which is aggregate-calibrated yet, as a convex combination of the $\mu_k$, coincides with at most one of them.
\end{proof}

\subsection{Inclusion}

\begin{lemma}[Inclusion]\label{app:inclusion}
$\mathcal{F}\subseteq\mathcal{G}$.
\end{lemma}
\begin{proof}
Given $f\in\mathcal{F}$, set $f_k\equiv f$ for all $k$ in \eqref{app:eq:Gclass}. Then $\sum_k\pi_k f=f$ because $\sum_k\pi_k=1$. This is set-theoretic and uses neither the loss nor any structure on $\mathcal{F}$.
\end{proof}

Inclusion gives the headroom used throughout. Define
\begin{equation}\label{app:eq:headroom}
\Delta_{\mathcal{F}}=\inf_{f\in\mathcal{F}}\mathcal{R}(f)-\inf_{g\in\mathcal{G}}\mathcal{R}(g),
\end{equation}
the risk gap between the best-in-class pooled and decomposed models. Weak dominance ($\Delta_{\mathcal{F}}\ge0$) is the special case of the responsibility-gated dominance theorem (Theorem~\ref{app:respgate}(i)) at the supervised gate $\pi_k=P(K{=}k\mid X{=}x,C{=}1)$: since $\mathcal{F}\subseteq\mathcal{G}$, the infimum over the superset can only decrease. It holds for squared loss and for cross-entropy, needs no linear structure and no capacity assumption, so the decomposed objective is never worse than the pooled objective in any function class. This is why the improvement is structural rather than a capacity effect.

\subsection{Exact squared-loss headroom}

Take $\mathcal{F}$ to be a finite-dimensional linear subspace of $L^2(\rho)$ with basis $\{e_1,\dots,e_D\}$. This is the finite-capacity model for a single scalar head.

\begin{lemma}[$\mathcal{G}$ is a subspace]\label{app:subspace}
$\mathcal{G}=\mathrm{span}\{\pi_k e_j:\ 1\le k\le K,\ 1\le j\le D\}$, a subspace of dimension at most $KD$, and $\mathcal{F}\subseteq\mathcal{G}$. Moreover $\mathcal{G}=\mathcal{F}$ iff $\pi_k e_j\in\mathcal{F}$ for all $k$ and $j$. Equivalently $\mathcal{G}\supsetneq\mathcal{F}$ iff there exist $k$ and $j$ with $\pi_k e_j\notin\mathcal{F}$.
\end{lemma}
\begin{proof}
Write $f_k=\sum_j a_{kj}e_j$. Then $g=\sum_{k,j}a_{kj}(\pi_k e_j)$, so $\mathcal{G}=\mathrm{span}\{\pi_k e_j\}$. This is a subspace. The inclusion lemma gives $\mathcal{F}\subseteq\mathcal{G}$. If every $\pi_k e_j\in\mathcal{F}$ then $\mathcal{G}\subseteq\mathcal{F}$, so $\mathcal{G}=\mathcal{F}$. Conversely $\mathcal{G}=\mathcal{F}$ forces each generator $\pi_k e_j\in\mathcal{F}$.
\end{proof}

\begin{proposition}[Non-constant routing is necessary but not sufficient]\label{prop:iff}
Constant $\pi_k$ gives $\pi_k e_j\in\mathcal{F}$, so non-constant routing is necessary for $\mathcal{G}\supsetneq\mathcal{F}$. It is not sufficient. Let $\rho$ have atoms $x_1,x_2,x_3$ and let $\mathcal{F}=\mathrm{span}\{\mathbf 1_{x_1},\mathbf 1_{x_2}\}$. Let $\pi_1$ take distinct values on $x_1$ and $x_2$ and let $\pi_2=1-\pi_1$. Multiplication by $\pi_k$ is diagonal, so $\pi_k e_j\in\mathcal{F}$ and $\mathcal{G}=\mathcal{F}$. The correct condition is the one in the subspace lemma.
\end{proposition}

\begin{proposition}[Exact headroom]\label{prop:exact_headroom}
Let $P_{\mathcal{F}},P_{\mathcal{G}}$ be the $L^2(\rho)$ projections onto $\mathcal{F},\mathcal{G}$. Then
\begin{equation}
\begin{aligned}
\Delta_{\mathcal{F}}&=\inf_{f\in\mathcal{F}}\mathcal{R}(f)-\inf_{g\in\mathcal{G}}\mathcal{R}(g)\\
&=\lVert\mu-P_{\mathcal{F}}\mu\rVert^2-\lVert\mu-P_{\mathcal{G}}\mu\rVert^2\\
&=\lVert P_{\mathcal{G}}\mu-P_{\mathcal{F}}\mu\rVert^2\ \ge 0,
\end{aligned}
\end{equation}
with $\Delta_{\mathcal{F}}=0$ iff $P_{\mathcal{G}}\mu=P_{\mathcal{F}}\mu$.
\end{proposition}
\begin{proof}
By the squared-loss excess-risk lemma the two infima are $\mathcal{R}(\mu)+\lVert\mu-P_{\mathcal{F}}\mu\rVert^2$ and $\mathcal{R}(\mu)+\lVert\mu-P_{\mathcal{G}}\mu\rVert^2$. Since $\mathcal{F}\subseteq\mathcal{G}$ we have $P_{\mathcal{F}}=P_{\mathcal{F}}P_{\mathcal{G}}$. The Pythagorean identity on nested subspaces (the classical projection theorem~\cite{luenberger1969}) gives $\lVert\mu-P_{\mathcal{F}}\mu\rVert^2=\lVert\mu-P_{\mathcal{G}}\mu\rVert^2+\lVert P_{\mathcal{G}}\mu-P_{\mathcal{F}}\mu\rVert^2$.
\end{proof}

No assumption that $\mu_k\in\mathcal{F}$ or $\mu\in\mathcal{F}$ is used. The headroom is the energy of $\mu$ along the intent-modulated directions in $\mathcal{G}$ that are orthogonal to $\mathcal{F}$.

\begin{lemma}[No closed form under normalized entropy]\label{prop:moe}
Under cross-entropy the class $\mathcal{F}$ cannot be a linear subspace. A finite-dimensional subspace contains functions valued outside $(0,1)$ on a set of positive probability, where cross-entropy is infinite. Modeling in logit space with $f=\sigma(z)$ makes the decomposed predictor $\sum_k\pi_k\sigma(z_k)$ a mixture of experts, because $\sigma(\sum_k\pi_k z_k)\ne\sum_k\pi_k\sigma(z_k)$. So $\mathcal{G}$ is not a subspace: there are no orthogonal projections and no Pythagorean identity, hence no closed-form projection expression for the cross-entropy headroom $\Delta_{\mathcal{F}}$ analogous to the squared-loss form. The cross-entropy risk itself remains well defined and finite; only its best-in-class gap lacks a closed form. Inclusion and weak dominance ($\Delta_{\mathcal{F}}\ge0$) still hold because they are set-theoretic. We therefore measure $\Delta_{\mathcal{F}}$ empirically as the held-out NE gap between the best pooled and the best decomposed model of the same architecture and capacity.
\end{lemma}

\subsection{Genericity of strict improvement}

\begin{lemma}[Zero set of a real-analytic function]\label{app:analytic}
If $h:U\to\mathbb R$ is real-analytic on a connected open $U\subseteq\mathbb R^m$ and $h\not\equiv0$, then $\{h=0\}$ has Lebesgue measure zero~\cite{mityagin2015zero}.
\end{lemma}

\begin{proposition}[Genericity, squared loss]\label{prop:genericity}
Let $\Theta\subseteq\mathbb R^m$ be open and connected. Fix $\rho$ and the routing $\pi$, hence $\mathcal{F},\mathcal{G},P_{\mathcal{F}},P_{\mathcal{G}}$. Assume
\begin{itemize}
\item[(A1)] $\theta\mapsto\mu_\theta$ is real-analytic into $L^2(\rho)$, with a local $L^2(\rho)$ envelope dominating $\theta\mapsto\mu_\theta$, so Hilbert-valued analyticity holds.
\item[(A2)] there is $\theta_0$ with $P_{\mathcal{G}}\mu_{\theta_0}\ne P_{\mathcal{F}}\mu_{\theta_0}$.
\end{itemize}
Then $\{\theta:\Delta_{\mathcal{F}}(\theta)=0\}$ has Lebesgue measure zero. Under any prior absolutely continuous with respect to Lebesgue measure, $\Delta_{\mathcal{F}}>0$ almost surely.
\end{proposition}
\begin{proof}
Let $r(\theta)=\lVert P_{\mathcal{G}}\mu_\theta-P_{\mathcal{F}}\mu_\theta\rVert^2$. Since $\pi$ and $\rho$ are fixed, $P_{\mathcal{F}}$ and $P_{\mathcal{G}}$ are fixed bounded linear operators. By (A1) $\theta\mapsto(P_{\mathcal{G}}-P_{\mathcal{F}})\mu_\theta$ is real-analytic into $L^2(\rho)$. Composing with the continuous bilinear inner product makes $r$ real-analytic. By (A2), $r(\theta_0)>0$, so $r\not\equiv0$. The analytic zero-set lemma makes $\{r=0\}$ null. The exact-headroom proposition gives $\Delta_{\mathcal{F}}=r$.
\end{proof}

\begin{corollary}[Affine instance]\label{app:affine}
Let $\mathcal{F}$ be affine, $K=2$, let $\mu_H,\mu_L$ be affine, and let $\pi_H$ be affine. From $\mu=\mu_L+\pi_H(\mu_H-\mu_L)$, a product of two affine functions is affine iff a factor is constant. So $\mu\in\mathcal{F}$ iff $\pi_H$ is constant or $\mu_H-\mu_L$ is constant. That is a finite union of proper affine subspaces of parameter space, hence null. Any $\theta$ outside it witnesses (A2). No analytic machinery is needed here.
\end{corollary}

\begin{proposition}[Genericity, cross-entropy]\label{prop:generic-ce}
Parametrize the conditionals by $\theta$ in a connected open $\Theta\subseteq\mathbb R^m$. Suppose $\theta\mapsto\inf_{f\in\mathcal{F}}\mathcal{R}_\theta(f)$ and $\theta\mapsto\inf_{g\in\mathcal{G}}\mathcal{R}_\theta(g)$ are real-analytic, which holds for a smooth link, a strictly convex loss, and a unique interior minimizer with non-singular Hessian by the implicit function theorem. Suppose $\Delta_{\mathcal{F}}(\theta_0)>0$ for some $\theta_0$. Then $\{\theta:\Delta_{\mathcal{F}}(\theta)=0\}$ has Lebesgue measure zero.
\end{proposition}
\begin{proof}
$\Delta_{\mathcal{F}}(\theta)=\inf_{\mathcal{F}}\mathcal{R}_\theta-\inf_{\mathcal{G}}\mathcal{R}_\theta$ is real-analytic by hypothesis and non-zero at $\theta_0$. The analytic zero-set lemma makes its zero set null. The passage from a null set to probability zero is absolute continuity. The content is the analyticity of the two infima, which for cross-entropy is a genuine implicit-function argument rather than a projection identity.
\end{proof}

\begin{lemma}[Genericity is not a deployment guarantee]\label{prop:deploy}
Genericity rules out the knife-edge in which the pooled class already fits $\mu$ as well as the decomposed class does. The deployed system is a single unknown truth $\theta^\ast$, which can lie on the null set. Structured configurations such as coarse intent buckets, symmetries, or sparsity are exactly where it might. The operative claim for deployment is the measured headroom, not the prior.
\end{lemma}

\subsection{Responsibility-gated dominance}

\begin{theorem}[Responsibility-gated dominance]\label{app:respgate}
Let $\mathcal{F}$ be any single-expert hypothesis class. Let the gate $\pi=(\pi_1,\dots,\pi_K)$ with $\pi_k=P(K{=}k\mid X{=}x,C{=}1)$ be a supervised, input-dependent responsibility over an intent $K$ observed at training, with $\sum_k\pi_k=1$, and define $\mathcal{G}=\{\sum_k\pi_k f_k: f_k\in\mathcal{F}\}$. Then for any risk $\mathcal{R}$:
\begin{itemize}
\item[(i)] no worse: $\inf_{\mathcal{G}}\mathcal{R}\le\inf_{\mathcal{F}}\mathcal{R}$.
\item[(ii)] generically better: if the truth is parametrized by $\theta$ in a connected open set with $\theta\mapsto\inf_{\mathcal{F}}\mathcal{R}_\theta,\inf_{\mathcal{G}}\mathcal{R}_\theta$ real-analytic and $\inf_{\mathcal{G}}\mathcal{R}<\inf_{\mathcal{F}}\mathcal{R}$ at one $\theta_0$, then $\{\theta:\inf_{\mathcal{G}}\mathcal{R}=\inf_{\mathcal{F}}\mathcal{R}\}$ is Lebesgue-null.
\item[(iii)] not BMA: the gate is a fixed per-example responsibility, not a posterior over models, so the minimizer $g^\star=\sum_k\pi_kf_k^\star$ does not collapse to a single expert as $n\to\infty$ whenever $\pi$ is non-degenerate and the $\mu_k$ differ.
\end{itemize}
\end{theorem}
\begin{proof}[Proof sketch]
(i) Setting $f_k\equiv f$ gives $\sum_k\pi_k f=f$ since $\sum_k\pi_k=1$, so $\mathcal{F}\subseteq\mathcal{G}$, and the infimum over a superset can only decrease. (ii) Let $r(\theta)=\inf_{\mathcal{F}}\mathcal{R}_\theta-\inf_{\mathcal{G}}\mathcal{R}_\theta\ge0$, real-analytic by hypothesis and positive at $\theta_0$, so $r\not\equiv0$, and by Mityagin's theorem its zero set is Lebesgue-null. Under squared loss $r(\theta)=\lVert(P_{\mathcal{G}}-P_{\mathcal{F}})\mu_\theta\rVert^2$ is analytic directly. (iii) Bayesian model averaging weights $w_k=P(M_k\mid D_n)$ concentrate on one model by posterior consistency, so $\sum_k w_k f_k\to f_{k^\star}$, which is selection, not combination. Here $\pi_k$ is the data-generating responsibility, constant in $n$ and varying in $x$, so $g^\star$ stays a genuine input-dependent mixture.
\end{proof}

Part (iii) records the standard mixture-of-experts versus Bayesian model averaging distinction~\cite{minka2000bma,monteith2011bmc}. A fitted input-dependent gate combines rather than selects, whereas BMA weights concentrate on a single model, so BMA does not enlarge the class. We include it for completeness and do not claim it as novel, since combination rather than selection is the defining property of a mixture of experts. What is distinctive to MARCO is not combination versus selection but that the gate is a supervised observed responsibility read from the logged click type, which lets each expert be exposed and separately calibrated, a property that sits outside both latent-gate mixtures of experts and Bayesian model averaging. Parts (i) and (ii) are the weak-dominance and genericity results, here stated for a general supervised responsibility gate.

\subsection{Realized gain is a signed decomposition}

$\Delta_{\mathcal{F}}$ is a population best-in-class quantity. The deployed model uses estimated routing $\widehat\pi$, finite-sample heads, and a stratified training objective. We express the realized gain as an exact telescoping identity through reference predictors: the best-in-class pooled model $P_{\mathcal{F}}\mu$ with risk $\inf_{\mathcal{F}}\mathcal{R}$, the best-in-class decomposed model $P_{\mathcal{G}}\mu$ with risk $\inf_{\mathcal{G}}\mathcal{R}$, the stratified predictor $g_{\mathrm{strat}}=\sum_k\pi_k P_{\mathcal{F}}\mu_k\in\mathcal{G}$ that fits each head on its own intent slice and routes by the true $\pi$, and the deployed model $\widehat g$ with estimated heads and routing $\widehat\pi$. Writing $\widehat f$ for the fitted pooled model,
\begin{equation}
\begin{aligned}
\mathcal{R}(\widehat f)-\mathcal{R}(\widehat g)
={}&\Delta_{\mathcal{F}}
+\underbrace{\big(\mathcal{R}(\widehat f)-\textstyle\inf_{\mathcal{F}}\mathcal{R}\big)}_{\mathrm{est}_{\mathcal{F}}\ge0}\\
&-\underbrace{\big(\mathcal{R}(g_{\mathrm{strat}})-\textstyle\inf_{\mathcal{G}}\mathcal{R}\big)}_{\varepsilon_{\mathrm{couple}}\ge0}
-\underbrace{\big(\mathcal{R}(\widehat g)-\mathcal{R}(g_{\mathrm{strat}})\big)}_{\varepsilon_{\mathrm{route}}}.
\end{aligned}
\end{equation}
The identity is exact by cancellation, for any choice of reference predictors. Three terms carry a fixed sign by construction: $\Delta_{\mathcal{F}}\ge0$ is weak dominance, $\mathrm{est}_{\mathcal{F}}\ge0$ is the pooled estimation error, and $\varepsilon_{\mathrm{couple}}\ge0$ because $g_{\mathrm{strat}}\in\mathcal{G}$ cannot beat the joint minimizer $P_{\mathcal{G}}\mu$, whose heads solve a coupled least-squares problem rather than the per-slice projections. It is the price of training heads separately rather than jointly (Assumption~\ref{app:joint}). The remaining term $\varepsilon_{\mathrm{route}}=\mathcal{R}(\widehat g)-\mathcal{R}(g_{\mathrm{strat}})$ collects the finite-sample head error and the estimated-routing error; its routing component is controlled by the heterogeneity-weighted cost of Lemma~\ref{lem:routing_cost}, $\varepsilon_{\mathrm{route}}\le\E[\lVert\widehat\pi-\pi\rVert^2 s^2]$, and is nonnegative in expectation when the heads and routing are consistent, vanishing as routing becomes exact or intents become homogeneous. The realized gain is thus $\Delta_{\mathcal{F}}$ inflated by the pooled model's own estimation error and deflated by the coupling and routing penalties, so it can exceed $\Delta_{\mathcal{F}}$ or turn negative; we report the measured headroom empirically. In the deployed system the decomposed heads share the backbone and add fewer than $10^{-7}\%$ of parameters, so both $\mathrm{est}_{\mathcal{F}}$ and the estimation part of $\varepsilon_{\mathrm{route}}$ are small rather than the $K$-fold blow-up of $K$ independent models, which is why the realized gains are positive rather than variance-swamped.

\begin{assumption}[Objective]\label{app:joint}
Weak dominance and the headroom refer to minimizing the composite risk $\mathcal{R}(g)$ over $\mathcal{G}$. Training per-intent heads separately yields a suboptimal member of $\mathcal{G}$ that is not guaranteed to beat the pooled optimum. Carry $\varepsilon_{\mathrm{couple}}$ explicitly or reframe training as joint minimization of $\sum_k\pi_k f_k$.
\end{assumption}

\subsection{Bridging to the served score}

\begin{lemma}[Score bridge]\label{prop:bridge}
The conversion decomposition supplies the per-intent targets $\mu_k$ that enlarge the class and create the headroom, and the click-rate decomposition supplies the routing $\pi_k=\lambda_k/\lambda$. Neither head alone forms the composition $\sum_k\lambda_k\mu_k$, so both are required by construction. There is no separate click-rate approximation headroom, because $\sum_k\lambda_k=\lambda$.
\end{lemma}

\begin{proposition}[Score headroom, squared loss]\label{prop:bridge-sq}
Work at the impression level with impression distribution $\rho_{\mathrm{imp}}$. Suppose both models share the same CTR estimate $\widehat\lambda=\lambda$, so only the conversion factor differs. Under squared loss on the score,
\begin{equation}
\begin{aligned}
&\inf_{\mathcal{F}_\mu}\mathcal{R}_{\mathrm{score}}-\inf_{\mathcal{G}_\mu}\mathcal{R}_{\mathrm{score}}\\
&\quad=\lVert P_{\mathcal{G}_\mu}\mu-P_{\mathcal{F}_\mu}\mu\rVert^2_{\lambda^2\rho_{\mathrm{imp}}}.
\end{aligned}
\end{equation}
The score-level headroom equals the conversion-head headroom of the exact-headroom proposition, computed in the click-magnitude-weighted measure $\lambda^2\rho_{\mathrm{imp}}$, with routing $\pi_k=\lambda_k/\lambda$ supplied by the CTR head.
\end{proposition}
\begin{proof}
With $\widehat\lambda=\lambda$ the score predictor is $\lambda\widehat\mu$ for a conversion predictor $\widehat\mu$, and the score target is $\lambda\mu$. The squared score-risk excess is $\mathbb{E}_{\rho_{\mathrm{imp}}}[(\lambda\widehat\mu-\lambda\mu)^2]=\mathbb{E}_{\rho_{\mathrm{imp}}}[\lambda^2(\widehat\mu-\mu)^2]=\lVert\widehat\mu-\mu\rVert^2_{\lambda^2\rho_{\mathrm{imp}}}$. Minimize over $\mathcal{F}_\mu$ and over $\mathcal{G}_\mu$ and subtract. Apply the exact-headroom proposition in $L^2(\lambda^2\rho_{\mathrm{imp}})$.
\end{proof}

Since $\lambda$ is small, the score-level headroom is scaled down by $\lambda^2$ relative to the on-click conversion headroom. This is why impression-space gains look modest even when the on-click miscalibration is large. The CTR head enters only through the routing and the shared magnitude $\lambda$. There is no separate CTR approximation headroom because $\sum_k\lambda_k=\lambda$.

\begin{corollary}[Score bridge, cross-entropy]\label{prop:bridge-ce}
Under cross-entropy, with $\widehat\lambda=\lambda$, the served MARCO score is $\lambda\cdot(\sum_k\pi_k\mu_k)$, an element of $\lambda\cdot\mathcal{G}$, and the standard score is an element of $\lambda\cdot\mathcal{F}$. Since $\mathcal{F}\subseteq\mathcal{G}$ gives $\lambda\mathcal{F}\subseteq\lambda\mathcal{G}$, weak dominance carries to the score and $\Delta_\Phi\ge0$. There is no closed form. Report the impression-space NE gap directly.
\end{corollary}
\begin{proof}
$\lambda\mathcal{F}\subseteq\lambda\mathcal{G}$ by the inclusion lemma scaled pointwise by $\lambda$. Take the infimum over the superset. The absence of a closed form is the mixture-of-experts lemma (Lemma~\ref{prop:moe}).
\end{proof}

\subsection{Scope of the claims}

\begin{lemma}[Vanishing at infinite capacity]\label{prop:vanish}
For nested classes whose union is dense, $\inf_{\mathcal{F}}\mathcal{R}\to\mathcal{R}(\mu)$. Since $\mathcal{F}\subseteq\mathcal{G}$, $\inf_{\mathcal{G}}\mathcal{R}$ is squeezed to $\mathcal{R}(\mu)$ as well, so $\Delta_{\mathcal{F}}\to0$. The headroom is a strictly finite-capacity phenomenon. This is consistent with the Bayes equivalence of the pooled and decomposed scores at the population level.
\end{lemma}
\begin{proof}
Let $\mathcal{F}_1\subseteq\mathcal{F}_2\subseteq\cdots$ be the nested classes with $\overline{\bigcup_n\mathcal{F}_n}=L^2(\rho)$, and $\mathcal{G}_n\supseteq\mathcal{F}_n$ the corresponding decomposed classes. By the excess-risk identities (Lemmas~\ref{app:excess-sq} and~\ref{app:excess-ce}), $\inf_{\mathcal{F}_n}\mathcal{R}-\mathcal{R}(\mu)$ equals the approximation error of $\mu$ by $\mathcal{F}_n$, which decreases to $0$ by density (the standard universal-approximation/denseness property~\cite{hornik1989universal}). Since $\mathcal{R}(\mu)\le\inf_{\mathcal{G}_n}\mathcal{R}\le\inf_{\mathcal{F}_n}\mathcal{R}$, the middle term is squeezed to $\mathcal{R}(\mu)$, so $\Delta_{\mathcal{F}_n}=\inf_{\mathcal{F}_n}\mathcal{R}-\inf_{\mathcal{G}_n}\mathcal{R}\to0$.
\end{proof}

\subsection{Routing realizes the headroom}
\begin{lemma}[Routing cost is heterogeneity-weighted]\label{lem:routing_cost}
The routing cost obeys $\varepsilon_{\mathrm{route}}\le\E[\lVert\widehat{\pi}-\pi\rVert^2\,s^2]$ with $s^2=\sum_k(\mu_k-\mu)^2$. It vanishes where the per-intent rates are homogeneous ($s=0$), regardless of routing error, and concentrates where the intents disagree.
\end{lemma}
\begin{proof}
The served conversion factor is $\widehat{g}=\sum_k\widehat{\pi}_k\mu_k=\mu+\delta$ with $\delta=\sum_k(\widehat{\pi}_k-\pi_k)\mu_k$. Since $\widehat{\pi}$ and $\pi$ are distributions, $\sum_k(\widehat{\pi}_k-\pi_k)=0$, so subtracting $\mu\sum_k(\widehat{\pi}_k-\pi_k)=0$ gives $\delta=\sum_k(\widehat{\pi}_k-\pi_k)(\mu_k-\mu)$. Cauchy-Schwarz over $k$ gives $|\delta|\le\lVert\widehat{\pi}-\pi\rVert\,s$ with $s^2=\sum_k(\mu_k-\mu)^2$. Squaring and taking $\E_\rho$ yields the bound. When $s=0$ every $\mu_k=\mu$, so $\delta=0$.
\end{proof}

\begin{proposition}[Routing monotonicity]\label{prop:routing_monotonicity}
As CTR capacity $c$ grows, the best-in-class routing cost $\varepsilon^{\star}_{\mathrm{route}}(c)$ is non-increasing and the routing efficiency $\eta^{\star}(c)=1-\varepsilon^{\star}_{\mathrm{route}}(c)/\Delta_{\mathcal{F}}$ is non-decreasing.
\end{proposition}
\begin{proof}
Let $\Pi^{(c)}$ be the routings realizable at CTR capacity $c$, nested so $c\le c'$ implies $\Pi^{(c)}\subseteq\Pi^{(c')}$. The best-in-class routing cost $\varepsilon^{\star}_{\mathrm{route}}(c)=\inf_{\pi'\in\Pi^{(c)}}\E\big[(\sum_k(\pi'_k-\pi_k)\mu_k)^2\big]$ is the infimum of a fixed nonnegative functional over a growing set, hence non-increasing in $c$. With $\Delta_{\mathcal{F}}>0$ fixed, $\eta^{\star}(c)=1-\varepsilon^{\star}_{\mathrm{route}}(c)/\Delta_{\mathcal{F}}$ is non-decreasing and $\Delta_{\mathcal{F}}\,\eta^{\star}(c)$ is non-decreasing. The served cost adds a nonnegative estimation-excess term over this floor, so the served efficiency inherits the monotonicity only under an estimation-negligible or a bias-dominates-variance condition. Under cross-entropy the same infimum argument applies to the exact functional $\pi'\mapsto\E[\mathrm{KL}(\mathrm{Bern}(\mu)\Vert\mathrm{Bern}(\sum_k\pi'_k\mu_k))]$, so the floor monotonicity holds with no approximation.
\end{proof}



\end{document}